\pdfoutput=1
\documentclass[10pt,journal,letterpaper]{IEEEtran}
\usepackage{amsmath,amssymb,amsfonts,mathtools}
\usepackage{algorithm,algorithmic}
\usepackage{array}
\usepackage[caption=false,font=footnotesize,labelfont=sf,textfont=sf]{subfig}
\usepackage{graphicx,cite,url}
\usepackage{needspace}
\newtheorem{problem}{Problem}
\newtheorem{proposition}{Proposition}
\newtheorem{remark}{Remark}
\begin{document}
\title{Localization-Uncertainty-Aware UAV Swarm Search for GNSS-Denied Maritime Surveillance}
\author{Siyuan Yang, Junning Lyu, Xiaoyu He, Lakmal Seneviratne and Shaoming He\IEEEauthorrefmark{1}
	\thanks{This work is supported by the National Key Research and Development Program under Grant No. 2022YFE0204400 and National Natural Science Foundation of China under Grant No. 52302449.}
	\thanks{Siyuan Yang is with the School of Aerospace Engineering, Beijing Institute of Technology, Beijing 100081, China (e-mail: siyuan.yang@bit.edu.cn).}
	\thanks{Junning Lyu is with the School of Aerospace Engineering, Beijing Institute of Technology, Beijing 100081, China (e-mail: 3120215027@bit.edu.cn).}
	\thanks{Xiaoyu He is with the School of Aerospace Engineering, Beijing Institute of Technology, Beijing 100081, China (e-mail: 3120230048@bit.edu.cn).}
	\thanks{Lakmal Seneviratne is Professor Emeritus with King's College London, London WC2R 2LS, United Kingdom (e-mail: Lakmal.seneviratne@kcl.ac.uk).}
	\thanks{Shaoming He is with the School of Aerospace Engineering, Beijing Institute of Technology, Beijing 100081, China (e-mail: shaoming.he@bit.edu.cn).}
	\thanks{\IEEEauthorrefmark{1}Corresponding author.}}

\maketitle
\begin{abstract}
This paper studies UAV swarm search in GNSS-denied maritime environments,
where inter-UAV ranging drives cooperative localization
and uncertainty propagates along the multi-hop reference chain,
affecting connectivity, collision avoidance, and search effectiveness.
We propose a control barrier function (CBF) framework using split covariance intersection (SCI) to fuse neighbor ranges and generate covariance-derived uncertainty scales for mapping and control.
Pairwise constraints incorporate inter-step relative displacement to establish continuous-time safe-set invariance across discrete control steps under stated containment, feasibility, and synchronization conditions.
A hierarchical Voronoi-target strategy with formation keeping supports
efficient forward-looking coverage.
Sequential local quadratic programs using preceding-UAV commands complete a 14-UAV search in 221.1~s with 0.45~m root-mean-square localization error while satisfying all separation and reference bounds.
Comparisons across search regions, swarm sizes, sampled robust CBF, localizability and formation targets, and model predictive control quantify speed--accuracy--margin trade-offs.
Tests on an NVIDIA Jetson Orin NX, a typical UAV onboard computer, verify sub-4~ms mean control cycles and sub-5~ms mean full per-UAV onboard cycles, while Gazebo quadrotor flights verify mission completion and constraint compliance under over-water multipath, beacon-UAV drift, wind turbulence, yaw errors, and communication faults.
\end{abstract}

\section{Introduction}
\label{sec:introduction}
\begin{figure}[!t]
\centering
\includegraphics[width=\columnwidth]{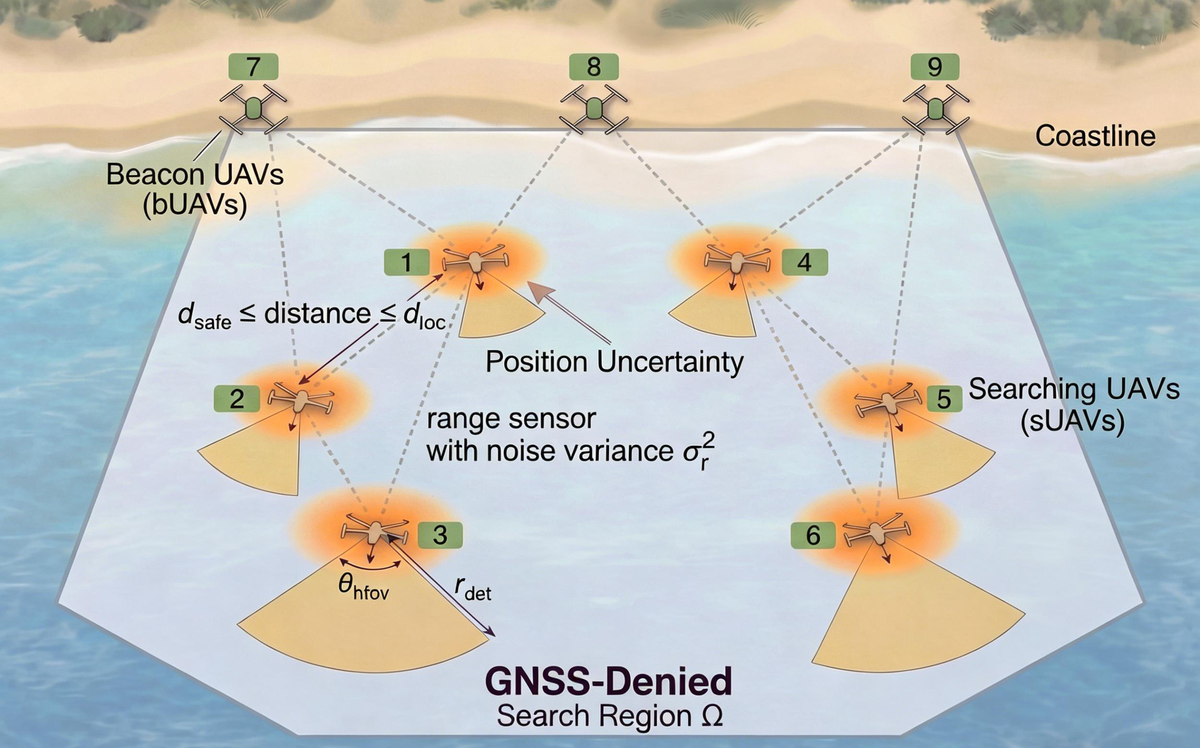}
\caption{The uGdMS scenario: sUAVs search a GNSS-denied region with forward-sector sensing while coastal bUAVs provide localization anchors.}
\label{fig:overview}
\end{figure}
\IEEEPARstart{U}{nmanned} aerial vehicle (UAV) swarms have received growing attention for search-and-rescue and surveillance because they cover large areas more efficiently than a single platform~\cite{bogdanowicz2017flying}.
In maritime surveillance, UAV swarms are attractive for search, border monitoring, and wide-area patrol over oceans.
Practical maritime operations are often challenged by Global Navigation Satellite System (GNSS)-denied conditions such as jamming, spoofing, and interference,
so UAVs must search, navigate, and coordinate with onboard sensing and inter-UAV cooperation.

This paper considers a GNSS-denied maritime surveillance scenario in which a swarm systematically searches a large ocean region extending from the coastline in Fig.~\ref{fig:overview}.
Hovering beacon UAVs (bUAVs) near the coast localize through coastal features such as lidar-visual-inertial odometry~\cite{shan_lvi-sam_2021},
while searching UAVs (sUAVs) perform coverage.
sUAVs carry forward-looking cameras or radar alongside ranging sensors for cooperative localization;
because direct ranging to coastal bUAVs cannot support long-range offshore search,
references propagate from the bUAVs into the search area in a multi-hop chain.
A central difficulty is that localization uncertainty accumulates along this chain:
each sUAV estimates its position from noisy ranges and from the already uncertain estimates of its reference neighbors,
which directly affects control design.
Specifically, localization-linked distances must remain within a chosen ranging-design limit after accounting for position errors,
and collision-avoidance constraints must include uncertainty margins for both UAVs.
Maritime search therefore requires an uncertainty-aware mechanism coordinating coverage, localization, and inter-UAV safety.

We refer to this problem as \emph{uncertainty-aware GNSS-denied maritime search (uGdMS)}.
Compared with conventional coverage control, uGdMS operates without global positioning,
its uncertainty is cascading because each UAV depends on the uncertain estimates of earlier reference neighbors,
and its forward-looking sensors couple position control with yaw orientation;
unlike our prior work under perfect localization~\cite{yang2024aerial},
we explicitly incorporate the cascading uncertainty into the control design.

Resolving this coupling requires positioning relay sUAVs so that they preserve ranging geometry toward the coast while advancing with the search frontier, and enforcing safety constraints despite estimation uncertainty and discrete-time actuation.
In contrast to approaches that optimize localization accuracy or fixed formations independently of evolving search targets, our framework integrates coast-anchored target geometry with sampled-data robust distance constraints informed by split covariance intersection (SCI) covariances.
The main contributions are as follows.
\begin{enumerate}
\item We organize sUAVs into squads whose leaders select targets through Voronoi partitioning over unsearched cells, while followers construct an adjustable triangular reference ladder between two coastal beacons and the leader's target. Fisher information matrix (FIM) analysis~\cite{cano_localizability_2023} characterizes how angular narrowing and offset adaptation balance ranging geometry against offshore reach under $d_{\rm loc}$.
\item Each sUAV tracks its target via a sequential local quadratic program (QP) using covariance-derived uncertainty scales and frozen same-step reference commands, constructing uncertainty-tightened control barrier functions (CBFs) for localization-connectivity maintenance and collision avoidance that explicitly bound inter-step displacement and establish conditional continuous-time safe-set invariance over each control interval $\Delta t$ under stated error-containment, feasibility, and synchronization conditions.
\item We evaluate the controller across region shapes, swarm sizes, and four adapted external baselines, and validate flight dynamics, physical maritime disturbances, communication faults, and real-time execution on a typical UAV onboard computer.
\end{enumerate}

\section{Related Work}
\label{sec:relatedwork}
Maritime search and rescue has attracted considerable attention~\cite{wu_maritime_rl_2024}.
Our previous work~\cite{yang2024aerial} assumed perfect localization. Here, sUAVs localize through range-based multi-hop cooperation, and the resulting uncertainty enters localization-connectivity maintenance, collision avoidance, and search control.
For long-range GNSS-denied operations over featureless waters, Peti \emph{et al.}~\cite{peti2025aerial} rely on long-range radio navigation for a single aircraft following preset waypoints, whereas the proposed framework coordinates a heterogeneous team in which searching UAVs also relay ranging references offshore.
Related range-aided coverage methods investigate polygonal disaster assessment under adaptive communication modes~\cite{wang_disaster_2026}, rather than extending a purely coast-anchored ranging chain across an unknown unsearched map.
Classical Voronoi coverage~\cite{cortes_coverage_2004,hu_voronoi_2020} provides principled spatial allocation, while related methods address connectivity maintenance~\cite{luo2019connectivity}.
Several works consider coverage control under localization uncertainty:
Papatheodorou \emph{et al.}~\cite{papatheodorou2018awgv} build guaranteed Voronoi cells from known circular regions but assume independent bounds and no connectivity maintenance;
Chen \emph{et al.}~\cite{chen2020collision} contract uncertain Voronoi cells with each sensor's own uncertainty, still with independent bounds;
Boldrer \emph{et al.}~\cite{boldrer2026uncertainty} achieve communication-less flocking without modeling error accumulation along reference chains;
and resilient coverage strategies~\cite{rezaee_uav_coverage_2024} likewise assume bounded independent uncertainty.
In contrast, this work models cascading localization covariance through SCI and incorporates the resulting radii into safety-critical control design, tying follower target geometry directly to both the leader's advancing coastal search frontier and the cascading localization scales.

Ultra-wideband (UWB)-based~\cite{xu_decentralized_2020} and multi-robot odometry-ranging~\cite{guo_formation_2020} methods are extensively studied for relative localization and formation control, but localization methods typically do not integrate uncertainty into safety-critical control design.
In active localization, Cano and Le Ny~\cite{cano_localizability_2023} shape follower trajectories to optimize Fisher-information-based localizability as leaders traverse predefined inspection waypoints, and Goudar and Schoellig~\cite{goudar_onboard_2026} integrate range-aided cooperative localization with distributed uncertainty-aware formation tracking.
Unlike navigation toward predetermined formations or waypoints, maritime coverage requires follower relays to adapt online to an evolving unsearched map without anchoring leaders near the coast.
On the estimation side, shared reference paths induce unknown cross-correlations among neighboring UAVs even on acyclic graphs. Covariance intersection~\cite{julier_unknown_correlations_1997} and distance-measurement SCI~\cite{cros_distance_2023} avoid double-counting shared information, known as data incest, while keeping fresh ranging noise independent, supplying the online fusion update adopted in this work.

Control barrier functions provide a systematic framework for enforcing state and input constraints in safety-critical systems~\cite{ames_control_2017}.
Recent advances extend CBFs to uncertain dynamics~\cite{garg_robust_2021,alan_cbf_iss_2023} and multi-UAV applications including collision avoidance~\cite{wang_multi-objective_2016} and coverage with connectivity maintenance~\cite{luo2019connectivity}, alongside sampled-data robust CBFs~\cite{oruganti_sampled_2024} and distributed robust tube model predictive control (MPC)~\cite{limon2010RobustTracking,nikou2019TubeDMPC}.
Existing CBF-based methods assume exact state information or externally specified, independent uncertainty bounds, ignoring uncertainty derived from cooperative localization, and most coverage-control formulations assume downward-looking sensors with circular footprints~\cite{funada_nonsmooth_cbf_2024}, while maritime surveillance typically uses forward-looking sensors with limited fields of view.
The robust CBF construction~\cite{garg_robust_2021} provides the uncertainty-margin machinery we adopt. Extending this construction, we derive covariance-informed signed barrier bounds online and include an explicit inter-step relative-motion margin on reference edges, enabling distributed sequential QPs to maintain continuous-time inter-sample invariance under stated containment and synchronization conditions.
The methodological differences from existing methods are summarized in Table~\ref{tab:comparison}.
\begin{table*}[!t]
\centering
\caption{Comparison of related works where $\checkmark$ denotes addressed, $\triangle$ denotes partially addressed, and $\times$ denotes unaddressed.}
\label{tab:comparison}
\begin{tabular}{l|cccc}
\hline
Feature & Prior work~\cite{yang2024aerial} & Cooperative Localization Works & Uncertainty-aware Coverage Works & This paper \\
\hline
GNSS-denied maritime search & $\checkmark$ & $\triangle$ & $\triangle$ & $\checkmark$ \\
Cooperative localization & $\times$ & $\checkmark$ & $\times$ & $\checkmark$ \\
Cascading uncertainty propagation & $\times$ & $\times$ & $\times$ & $\checkmark$ \\
Robust connectivity maintenance & $\times$ & $\triangle$ & $\times$ & $\checkmark$ \\
Collision avoidance under uncertainty & $\times$ & $\triangle$ & $\triangle$ & $\checkmark$ \\
Forward-looking sensing & $\times$ & $\times$ & $\times$ & $\checkmark$ \\
\hline
\end{tabular}
\end{table*}

\section{Preliminaries and Problem Formulation}
\label{sec:model}
\subsection{System and Cooperative Localization Model}
\label{sec:estimation}
Consider a swarm of $N$~sUAVs and $M$~bUAVs in a GNSS-denied environment, with a fixed altitude and synchronized nominal communication, so search and localization are planar.
Let $\mathcal S=\{1,\ldots,N\}$ be the controlled sUAVs and $\mathcal B$ the stationary coastal beacon UAVs with known coordinates.
In the nominal cooperative architecture, a long-range inter-UAV data link spans the coastal search region separately from the shorter-range physical ranging sensor, enabling position estimates, covariances, uncertainty scales, yaw measurements, and task-map updates to be exchanged at each synchronized step. These nominal modeling assumptions are adopted for controller synthesis and relaxed in Section~\ref{sec:gazebo}.

The planar position $\boldsymbol p_i\in\mathbb R^2$ and yaw angle $\psi_i$ of sUAV $i\in\mathcal S$ evolve under piecewise-constant commands applied over each control time step $t\in[t_k,t_k+\Delta t]$:
\begin{equation}
\dot{\boldsymbol p}_i=\boldsymbol v_i^k,\qquad
\dot\psi_i=\omega_i^k,\quad t\in[t_k,t_k+\Delta t],
\label{eq:dynamics}
\end{equation}
where $\boldsymbol v_i^k=[v_{i,x}^k,v_{i,y}^k]^T$ and $\omega_i^k$ denote the velocity and yaw-rate commands, subject to componentwise velocity bounds $|v_{i,x}|,|v_{i,y}|\le U_i$ and yaw-rate bound $|\omega_i|\le\bar\omega_i$. Equation~\eqref{eq:dynamics} models idealized first-order kinematics for outer-loop control design. Full second-order quadrotor flight dynamics, wind turbulence, and inner-loop tracking lag are evaluated in Section~\ref{sec:gazebo}.

Each UAV carries four sensors: a front-sector detection sensor, an altimeter, an inertial measurement unit (IMU), and a ranging sensor.
The IMU provides the yaw angle $\psi_i$ and the altimeter regulates a nominal 200~m flight altitude. Unlike the downward circular sensing of prior work~\cite{yang2024aerial}, the front-sector model accounts for yaw.
Modeling a directional camera or sector radar, the sensor in Fig.~\ref{fig:front-sector}(a) covers a sector of radius $r_{\rm det}$ and half-angle $\varphi$, giving full horizontal field of view $\theta_{\rm hfov}=2\varphi$.
A point $\boldsymbol m$ in the surface grid is considered searched by UAV $i$ if it satisfies both distance and angular constraints:
\begin{equation}
\|\boldsymbol m-\boldsymbol p_i\|\le r_{\rm det},\quad
\gamma_i(\boldsymbol m,\boldsymbol p_i)\le\varphi,
\label{eq:sector_sensing}
\end{equation}
where $\gamma_i(\boldsymbol m,\boldsymbol p)\in[0,\pi]$ denotes the angle between the sensor boresight and $\boldsymbol m-\boldsymbol p$.

\begin{figure}[!t]
\centering
\subfloat[]{\includegraphics[width=0.54\columnwidth]{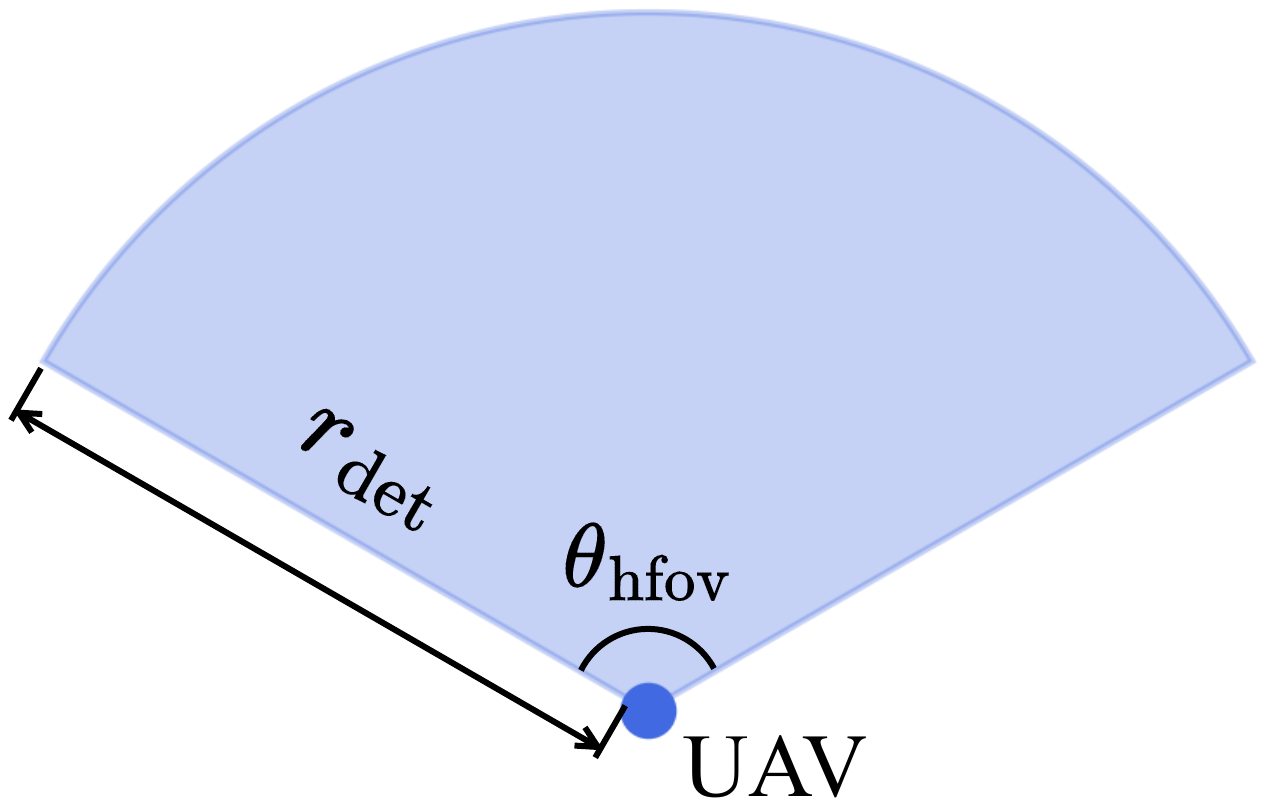}}
\hfill
\subfloat[]{\includegraphics[width=0.38\columnwidth]{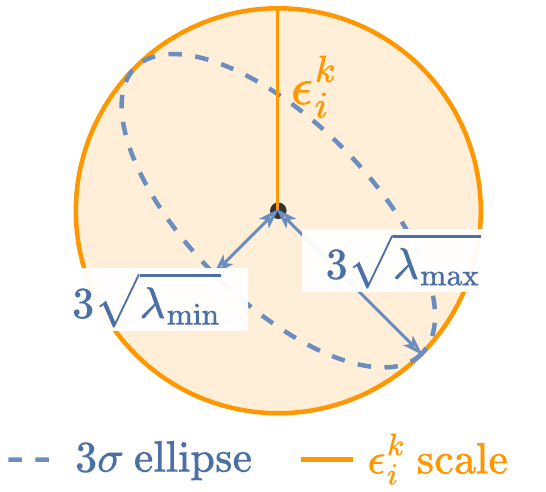}}
\caption{(a) Top-down view of the front-sector sensing model with range $r_{\rm det}$ and horizontal field of view $\theta_{\rm hfov}=2\varphi$; (b) the uncertainty radius $\epsilon_i^k=3\sqrt{\operatorname{tr}\boldsymbol P_i^k}$ encloses the Mahalanobis-radius-3 ellipse of the position covariance.}
\label{fig:front-sector}
\end{figure}

The ranging measurement between UAV $i$ and neighbor $j$ is:
\begin{equation}
z_{ij}=\|\boldsymbol p_i-\boldsymbol p_j\|+\eta_{ij},
\label{eq:ranging}
\end{equation}
where $\eta_{ij}\sim\mathcal N(0,\sigma_{ij}^2)$ is ranging noise.
Consistent with distance-dependent UWB ranging error characteristics~\cite{bellusci_uwb_2009}, link availability probability and noise standard deviation vary with true range $d=\|\boldsymbol p_i-\boldsymbol p_j\|$ as
\begin{equation}
P_{\rm avail}(d)=\frac{1}{1+(d/d_{\rm r})^2},\qquad
\sigma(d)=\sigma_{0}\left[1+(d/d_{\rm r})^2\right],
\label{eq:ranging_model}
\end{equation}
where $d_{\rm r}$ is the ranging distance scale and $\sigma_0$ is the zero-range noise standard deviation.
To model range-dependent over-water attenuation, we pair the control reference-distance limit $d_{\rm loc}=850$~m with an operational ranging cutoff $d_{\max}=1200$~m beyond which sparse, heavily attenuated links are discarded.
Equation~\eqref{eq:ranging_model} defines the nominal observation model, whereas Section~\ref{sec:gazebo} further examines degraded ranging parameters, sea-surface multipath bias, outages, and bUAV drift.

To extend the operating range, sUAVs localize through a multi-hop cascade in which each sUAV processes available measurements from all coastal beacons and strictly lower-indexed peers, including those in other squads, forming a causal estimation graph broader than its two fixed assigned control references $\mathcal R_i^{\rm loc}$ in Fig.~\ref{fig:reference-target}.
Within each squad $g$, members are indexed $\ell=1,\ldots,n_g$ from coast to frontier, with $\ell=n_g$ designating the squad leader.

Between measurement updates, the prediction follows the active velocity command $\boldsymbol v_i^{k-1}$ from step $k-1$:
\begin{equation}
\hat{\boldsymbol p}_i^- = \hat{\boldsymbol p}_i+\Delta t\boldsymbol v_i^{k-1},\qquad
\boldsymbol P_i^- = \boldsymbol P_i+\sigma_{\rm w}^2\Delta t^2\boldsymbol I,
\label{eq:ekf_prediction}
\end{equation}
where $\sigma_{\rm w}$ is the process-noise speed.
For bUAVs, the position is known, with $\hat{\boldsymbol p}_b=\boldsymbol p_b$ and $\boldsymbol P_b=\boldsymbol 0$ nominally, and sUAVs update in increasing index order as causal ranges arrive.
For an incoming range from $j$ to $i$, define unit bearing $\boldsymbol n=(\hat{\boldsymbol p}_i-\hat{\boldsymbol p}_j)/\|\hat{\boldsymbol p}_i-\hat{\boldsymbol p}_j\|$, projections $a=\boldsymbol n^T\boldsymbol P_i\boldsymbol n$, $b=\boldsymbol n^T\boldsymbol P_j\boldsymbol n$, and measurement variance $r=\sigma_{ij}^2$.
With estimation errors $\boldsymbol e_i=\boldsymbol p_i-\hat{\boldsymbol p}_i$, the linearized innovation is
\begin{equation}
\nu=z_{ij}-\|\hat{\boldsymbol p}_i-\hat{\boldsymbol p}_j\|\approx
\boldsymbol n^T(\boldsymbol e_i-\boldsymbol e_j)+\eta_{ij}.
\label{eq:innovation_approximation}
\end{equation}

Shared reference ancestors correlate $\boldsymbol e_i$ and $\boldsymbol e_j$ even on directed acyclic graphs (DAGs), inducing data incest and rumor propagation unless explicitly mitigated.
Following the distance-measurement SCI rule~\cite{cros_distance_2023} rooted in covariance intersection~\cite{julier_unknown_correlations_1997}, for weight $\omega\in(0,1)$ we compute
\begin{equation}
\begin{aligned}
 d(\omega)&=\omega a+(1-\omega)(b+\omega r),\\
 \boldsymbol K(\omega)&=\frac{\omega\boldsymbol P_i\boldsymbol n}{d(\omega)},\qquad
 \hat{\boldsymbol p}_i^+=\hat{\boldsymbol p}_i+\boldsymbol K\nu,\\
 \boldsymbol P_i^+(\omega)&=\left[(1-\omega)\boldsymbol P_i^{-1}
 +\frac{\omega}{b+\omega r}\boldsymbol n\boldsymbol n^T\right]^{-1}.
\end{aligned}
\label{eq:sci_update}
\end{equation}

For numerical stability, Eq.~\eqref{eq:sci_update} is computed via the Joseph form
\begin{equation}
\boldsymbol P_i^+=\frac{(\boldsymbol I-\boldsymbol K\boldsymbol n^T)\boldsymbol P_i
(\boldsymbol I-\boldsymbol K\boldsymbol n^T)^T}{1-\omega}
+\left(\frac b\omega+r\right)\boldsymbol K\boldsymbol K^T,
\label{eq:sci_joseph}
\end{equation}
with $\omega\in(0,1)$ chosen to minimize $\operatorname{tr}\boldsymbol P_i^+$. If correlated neighbor uncertainty outweighs ranging information, the estimator retains the prior $\boldsymbol P_i^+=\boldsymbol P_i$.
For a beacon with $b=0$, the optimal weight approaches zero and Eqs.~\eqref{eq:sci_update}--\eqref{eq:sci_joseph} reduce to the standard Kalman update $\boldsymbol K=\boldsymbol P_i\boldsymbol n/(a+r)$.
Range innovations outside $|\nu|\le 3\sqrt{(\sqrt a+\sqrt b)^2+r}$ are rejected against the bounded projected variance under arbitrary correlation without inflating the measurement variance $r$.
Under consistent priors, local linearization, and independent measurement noise, this update preserves posterior covariance consistency without zeroing unknown cross-correlations.

Indexing sUAVs by $i\in\mathcal S$ and coastal bUAVs by $b\in\mathcal B$, we define the scalar uncertainty scale
\begin{equation}
\epsilon_i^k=3\sqrt{\operatorname{tr}\boldsymbol P_i^k},\qquad
\epsilon_b^k=0,\qquad \epsilon_{ij}^k=\epsilon_i^k+\epsilon_j^k,
\label{eq:node_scale}
\end{equation}
where $\epsilon_b^k=0$ reflects the nominal stationary known-beacon condition $\boldsymbol P_b=\boldsymbol 0$ and is relaxed to nonzero $\boldsymbol P_b\succ\boldsymbol 0$ and $\epsilon_b>0$ under coastal bUAV drift in Section~\ref{sec:gazebo}.
Because $\lambda_{\max}(\boldsymbol P_i^k)\le\operatorname{tr}\boldsymbol P_i^k$, the circular radius $\epsilon_i^k$ encloses the Mahalanobis-radius-3 ellipse shown in Fig.~\ref{fig:front-sector}(b).
Under a zero-mean 2-D Gaussian error with covariance $\boldsymbol P_i^k$, that enclosed $3\sigma$ ellipse captures $1-e^{-9/2}\approx 98.89\%$ of probability mass, so $\epsilon_i^k$ serves as a covariance-derived probabilistic design scale rather than a deterministic bound on realized estimator errors.
Section~\ref{sec:preservation} therefore treats pairwise containment $\|\boldsymbol r_{ij}(t_k)-\hat{\boldsymbol r}_{ij}^k\|\le \epsilon_{ij}^k$ as a conditional hypothesis, implied via the triangle inequality when both endpoints satisfy individual containment, and Sections~\ref{sec:results}--\ref{sec:gazebo} track the per-sUAV empirical containment rate $C_{\rm env}$.

Coverage targets are selected using a shared online task map built strictly from estimated poses and scales $\epsilon_i^k$, while a separate ground-truth map records true sensor coverage for performance evaluation.
Let $d_i=\|\boldsymbol m-\hat{\boldsymbol p}_i\|$ and $\bar\gamma_i=\gamma_i(\boldsymbol m,\hat{\boldsymbol p}_i)$.
On the online map, point $\boldsymbol m$ is marked visited only when every position within $\epsilon_i^k$ of $\hat{\boldsymbol p}_i$ covers $\boldsymbol m$:
\begin{equation}
 d_i>\epsilon_i^k,\quad d_i+\epsilon_i^k\le r_{\rm det},\quad
 \bar\gamma_i+\arcsin(\epsilon_i^k/d_i)\le\varphi,
\label{eq:shrunk_sector}
\end{equation}
taking the cumulative union over all sUAVs across time.

\subsection{Problem Formulation and Applicability Scope}
\label{sec:formulation}
Let $\Omega\subset\mathbb R^2$ denote the convex coastal search region discretized into a finite grid $\mathcal G$, and define indicator $\chi_k(\boldsymbol m)=1$ if $\boldsymbol m\in\mathcal G$ remains unsearched on the online map at step $k$, yielding the unsearched count
\begin{equation}
\mathcal{U}_k=\sum_{\boldsymbol m\in\mathcal G}\chi_k(\boldsymbol m).
\label{eq:unsearched_count}
\end{equation}
\begin{problem}[Uncertainty-aware GNSS-denied maritime search (uGdMS)]
Given a swarm of $N$~sUAVs and $M$~bUAVs, search region $\Omega$, grid $\mathcal G$, assigned references $\mathcal R_i^{\rm loc}$, sensing parameters, input bounds, maximum localization range $d_{\rm loc}$, and minimum safe separation distance $d_{\rm safe}$, design decentralized control inputs driving $\mathcal U_k$ toward zero while the following conditions are maintained:
\begin{align}
\|\boldsymbol p_i-\boldsymbol p_j\|&\ge d_{\rm safe},
 &i,j\in\mathcal S,\ i\ne j,\\
\|\boldsymbol p_i-\boldsymbol p_j\|&\le d_{\rm loc},
 &i\in\mathcal S,\ j\in\mathcal R_i^{\rm loc},
\label{eq:problem_distances}
\end{align}
together with the actuator bounds in Eq.~\eqref{eq:dynamics}.
\end{problem}

\begin{remark}[Applicability scope]
\label{rem:applicability}
While motivated by GNSS-denied offshore surveillance, the core modules, namely causal multi-hop reference assignment, adjustable coast-anchored ladder targets, and covariance-scaled distance constraints, transfer directly to planar terrestrial coverage or open-area disaster assessment supported by surviving fixed anchors and pairwise ranging.
Applying the architecture to cluttered urban canyons, three-dimensional terrain following, or nonconvex corridors requires extending the visibility, obstacle-avoidance, and workspace-partitioning models accordingly.
\end{remark}

\section{uGdMS Solution}
\label{sec:targets}
To solve Problem 1, we coordinate three elements across each control step: a coastal-anchored reference ladder assigning two fixed ranging neighbors per follower, hierarchical Voronoi and ladder targets coupling exploration with relay geometry, and distributed sequential local QPs embedding inter-step displacement bounds.

\subsection{Reference Chain and Hierarchical Target Construction}
\label{sec:hierarchy}
\label{sec:geometry}
We adopt a \emph{triangular ladder} topology, which balances localization robustness and formation simplicity: its fixed CBF references create overlapping triangular units propagating offshore from the coastal bUAVs.
Each squad $g$ is assigned an ordered bUAV pair $(b_g^{\rm out},b_g^{\rm shared})$, and neighboring squads may share one bUAV.
Member $\ell\in\{1,\ldots,n_g\}$ maintains two fixed assigned CBF references
\begin{equation}
\mathcal R_{g,\ell}=\begin{cases}
\{b_g^{\rm out},b_g^{\rm shared}\},&\ell=1,\\
\{(g,1),b_g^{\rm out}\},&\ell=2,\\
\{(g,\ell-1),(g,\ell-2)\},&\ell\ge3,
\end{cases}
\label{eq:reference_set}
\end{equation}
so $\mathcal R_{g,\ell}$ instantiates $\mathcal R_i^{\rm loc}$, where $\{(g,\ell-1),(g,\ell-2)\}$ are the local indices of the two preceding UAVs in the same squad.
Ordering bUAVs first and then sUAVs by squad-local index gives a common topological ordering, so every active localization graph is a DAG rooted at the bUAVs, as illustrated for two seven-member squads in Fig.~\ref{fig:reference-target}.
Assigning two noncollinear planar references per follower establishes localizability while preserving chain mobility. By contrast, imposing three or more constrained edges tightens the feasible polytope and curtails offshore reach.

\begin{figure}[!t]
\centering
\includegraphics[width=57mm]{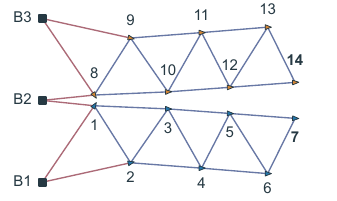}
\caption{Assigned ranging references in the 14-sUAV configuration. Blue and orange distinguish squads, squares represent fixed coastal beacons, and sUAVs 7 and 14 serve as squad leaders.}
\label{fig:reference-target}
\end{figure}

Target selection is hierarchical: at each update, squad leaders partition $\Omega$ through a centroidal Voronoi tessellation (CVT)-inspired construction~\cite{cortes_coverage_2004,hu_voronoi_2020} evaluated at estimated positions $\hat{\boldsymbol p}_i$, and select the unsearched grid point $\boldsymbol m$ satisfying $\chi_k(\boldsymbol m)=1$ inside their cell nearest to the center of the leader's field of view, or the geometric centroid once covered, as goal $\boldsymbol G$. Followers then form a triangular ladder extending from the bUAV side toward the leader target.
Specifically, followers construct a dynamic ladder connecting shared beacon $\boldsymbol B$ and outer beacon $\boldsymbol C$ toward $\boldsymbol G$ in Fig.~\ref{fig:target-geometry}.
For reach distance $L=\|\boldsymbol G-\boldsymbol B\|>0$, let axis unit vector $\boldsymbol e=(\boldsymbol G-\boldsymbol B)/L$ and orthogonal unit vector $\boldsymbol n$ point into the half-plane containing $\boldsymbol C$.
For a seven-member squad, initial coastal offset $a$ and axial step $s$ follow
\begin{equation}
 a=\frac{\min\{300,L/2\}+L/4}{2},\qquad s=\frac{L-a}{3}.
\label{eq:spacing}
\end{equation}

With lateral amplitude $h=(s/2)\tan\theta$ and longitudinal offset $d\in[0,s/2]$, squad targets sit at
\begin{equation}
\begin{aligned}
\boldsymbol q_{2j+1}&=\boldsymbol B+(a+js)\boldsymbol e, &&j=0,\ldots,3,\\
\boldsymbol q_{2j+2}&=\boldsymbol B+(a+js+d)\boldsymbol e+h\boldsymbol n, &&j=0,\ldots,2,
\end{aligned}
\label{eq:ladder_targets}
\end{equation}
assigning each member $\ell$ to $\boldsymbol q_\ell$ with $\boldsymbol q_7=\boldsymbol G$.
Because fixed lateral geometry at large $L$ would place nominal target links beyond $d_{\rm loc}$, active distance constraints would arrest follower and leader motion well short of the goal.
To sustain offshore mobility, the ladder narrows triangle angle $\theta$ smoothly with coastal distance $L$ in Fig.~\ref{fig:target-geometry}(b),
\begin{equation}
 \theta=\begin{cases}
 60^\circ, & L\le 2400,\\
 60^\circ-15^\circ\dfrac{L-2400}{800}, & 2400<L<3200,\\
 45^\circ, & L\ge 3200,
 \end{cases}
\label{eq:narrow}
\end{equation}
where $L$ is in metres, and optimizes $d$ near the shore so sUAV 2 respects limit $d_{\rm loc}$ relative to both beacon $\boldsymbol C$ and sUAV 3 in Fig.~\ref{fig:target-geometry}(c):
\begin{equation}
\begin{aligned}
d^*=\underset{0\le d\le s/2}{\operatorname{arg\,max}}\quad &d\\
\text{subject to}\quad
&\|\boldsymbol q_2(d)-\boldsymbol C\|\le d_{\rm loc},\\
&\|\boldsymbol q_2(d)-\boldsymbol q_3\|\le d_{\rm loc}.
\end{aligned}
\label{eq:target_offset}
\end{equation}

For extreme frontier goals where even the narrowed ladder spans exceed $d_{\rm loc}$, $\boldsymbol q_\ell$ acts as nominal guidance while the hard CBF constraints of Section~\ref{sec:qp} maintain true physical distances within $d_{\rm loc}$.
The five-member squad variant uses the same construction with $s=(L-a)/2$ for the two-squad 10-sUAV deployment, while the three-squad 21-sUAV deployment adopts seven-member squads with four coastal beacons.

\begin{figure}[!t]
\centering
\subfloat[]{\includegraphics[width=0.95\columnwidth]{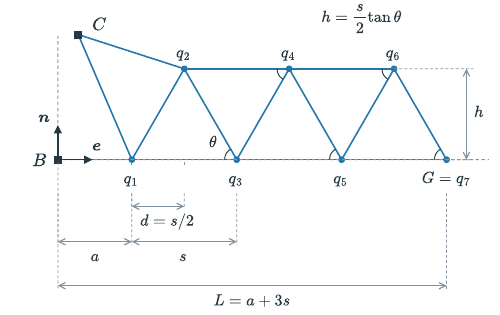}}\\
\subfloat[]{\includegraphics[width=0.95\columnwidth]{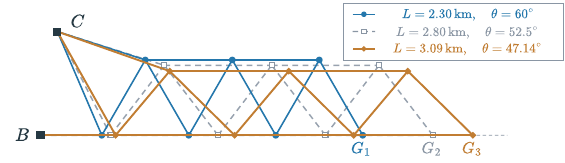}}\\
\subfloat[]{\includegraphics[width=0.95\columnwidth]{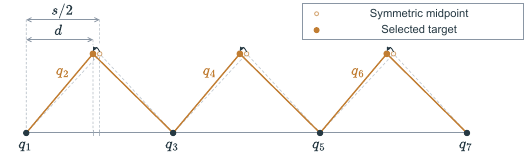}}
\caption{Follower target ladder construction: (a) symmetric nominal ladder with $\theta=60^\circ$ and $d=s/2$; (b) angular narrowing from $60^\circ$ to $45^\circ$ as coastal reach $L$ increases; (c) longitudinal offset pullback $d^*$ near the beacons.}
\label{fig:target-geometry}
\end{figure}

The local geometry of Eq.~\eqref{eq:ladder_targets} governs single-step localization conditioning.
Considering two known reference anchors $\boldsymbol a_1,\boldsymbol a_2$ with independent range precisions $w_j=\sigma_j^{-2}$ and unit bearings $\boldsymbol n_j=(\boldsymbol p-\boldsymbol a_j)/\|\boldsymbol p-\boldsymbol a_j\|$, the conditional FIM and its determinant satisfy~\cite{cano_localizability_2023}
\begin{equation}
\boldsymbol I=\sum_{j=1}^{2}w_j\boldsymbol n_j\boldsymbol n_j^T,
\qquad \det\boldsymbol I=w_1w_2\sin^2\vartheta,
\label{eq:conditional_fim}
\end{equation}
where $\vartheta$ is the intersection angle between $\boldsymbol n_1$ and $\boldsymbol n_2$, making $\boldsymbol I^{-1}$ the Cram\'er--Rao lower bound (CRLB) for single-sUAV unbiased estimation.
For interior ladder targets $\ell\ge 3$ referenced to $\ell-1$ and $\ell-2$, the reference intersection angle equals $\operatorname{atan2}(h,s-d)$ at axial positions and $\operatorname{atan2}(h,d)$ at offset positions, simplifying to $\vartheta=\theta$ with $\det\boldsymbol I=w^2\sin^2\theta$ under symmetric midpoints, where $d=s/2$ and $w_1=w_2=w$.
Across the $\theta\in[45^\circ,60^\circ]$ design span, Eq.~\eqref{eq:conditional_fim} decreases smoothly as $\sin^2\theta$ drops from $0.75$ to $0.50$, quantifying the controlled trade-off between instantaneous triangular conditioning and extended coastal reach within limit $d_{\rm loc}$.

\subsection{Distributed Robust CBF Controller}
\label{sec:qp}
\label{sec:robust}
We now incorporate the uncertainty radii $\epsilon_i^k$ into the CBF constraint design for localization-connectivity maintenance at $d_{\rm loc}$ and collision avoidance at $d_{\rm safe}$, using~squared-distance barriers for $\boldsymbol r_{ij}=\boldsymbol p_i-\boldsymbol p_j$:
\begin{equation}
h_{ij}^{\rm c}=\|\boldsymbol r_{ij}\|^2-d_{\rm safe}^2,\qquad
h_{ij}^{\rm l}=d_{\rm loc}^2-\|\boldsymbol r_{ij}\|^2,
\label{eq:true_distance_barriers}
\end{equation}
with time derivatives under Eq.~\eqref{eq:dynamics} given by
\begin{equation}
\dot h_{ij}^{\rm c}=2\boldsymbol r_{ij}^T(\boldsymbol v_i-\boldsymbol v_j),\qquad
\dot h_{ij}^{\rm l}=-2\boldsymbol r_{ij}^T(\boldsymbol v_i-\boldsymbol v_j).
\label{eq:distance_barrier_derivatives}
\end{equation}

Following the robust CBF construction~\cite{garg_robust_2021}, at step $k$ let $\hat{\boldsymbol r}_{ij}^k=\hat{\boldsymbol p}_i^k-\hat{\boldsymbol p}_j^k$, $s_{ij}=\|\hat{\boldsymbol r}_{ij}^k\|$, and $\epsilon_{ij}^k=\epsilon_i^k+\epsilon_j^k$.
Whenever actual relative error satisfies $\|\boldsymbol r_{ij}-\hat{\boldsymbol r}_{ij}^k\|\le \epsilon_{ij}^k$, assuming $s_{ij}\ge \epsilon_{ij}^k$ on collision pairs, the triangle inequality provides the signed lower barrier bounds
\begin{equation}
\begin{aligned}
 \underline h_{ij}^{\rm c}&=(s_{ij}-\epsilon_{ij}^k)^2-d_{\rm safe}^2,\\
 \underline h_{ij}^{\rm l}&=d_{\rm loc}^2-(s_{ij}+\epsilon_{ij}^k)^2,
\end{aligned}
\label{eq:snapshot_barriers}
\end{equation}
retaining negative values to exert restorative pull back into the safe set whenever triggered.

Let $\boldsymbol w_{ij}=\boldsymbol v_i-\boldsymbol v_j$, $L_{ij}=\|\boldsymbol w_{ij}\|_1$, and $B_{ij}=U_i+U_j$, where $U_j=0$ for bUAVs.
For positive CBF gains $\alpha_{\rm c},\alpha_{\rm l}>0$, we impose the sampled-data hard inequalities
\begin{align}
 F_{ij}^{\rm c}&=2(\hat{\boldsymbol r}_{ij}^k)^T\boldsymbol w_{ij}
 -2\epsilon_{ij}^kL_{ij}+\alpha_{\rm c}\underline h_{ij}^{\rm c}\ge0,
 \label{eq:collision_row}\\
 F_{ij}^{\rm l}&=-2(\hat{\boldsymbol r}_{ij}^k)^T\boldsymbol w_{ij}
 -(2\epsilon_{ij}^k+\Delta tB_{ij})L_{ij}+\alpha_{\rm l}\underline h_{ij}^{\rm l}\ge0.
 \label{eq:link_row}
\end{align}

Here $2\epsilon_{ij}^kL_{ij}$ absorbs uncertainty in the initial barrier gradient due to position error, whereas the additional reference term $\Delta t B_{ij}L_{ij}$ bounds inward curvature over the discrete control step $[t_k,t_k+\Delta t]$.
With these directed rows, each unordered pair $\{i,j\}$, $j<i$, is hard-enforced strictly by sUAV $i$'s local QP via Eq.~\eqref{eq:collision_row}, and each sUAV enforces reference rows Eq.~\eqref{eq:link_row} toward its lower-indexed assigned pair $\mathcal R_i^{\rm loc}$.
Solving sequentially in topological order $i=1,\ldots,N$ ensures sUAV $i$ treats already computed preceding velocities $\boldsymbol v_j^k$ for $j<i$ from the same step $k$ as fixed known vectors.

Given target position $\boldsymbol q_i$, task CBFs guide each UAV toward its target, decomposing distance and yaw objectives to eliminate orbiting:
\begin{equation}
 h_{i,d}=-k_p\|\boldsymbol q_i-\hat{\boldsymbol p}_i\|,\quad
 h_{i,\psi}=\tfrac{k_\psi}{2}[\cos(\psi_i^d-\psi_i)-1],
\end{equation}
where $\psi_i^d=\angle(\boldsymbol q_i,\hat{\boldsymbol p}_i)$ is the desired heading from the estimated position to the target.
The constraint CBFs are hard inequalities, while the task CBFs are softened by nonnegative slack variables $\delta_{i,\lambda}\ge 0$ for $\lambda\in\{d,\psi\}$, yielding linearized rows $\boldsymbol a_{i,\lambda}^T\boldsymbol u_i+c_{i,\lambda}\ge-\delta_{i,\lambda}$, where $c_{i,\lambda}=\alpha_\lambda h_{i,\lambda}$ with position and yaw task CBF gains $\alpha_d,\alpha_\psi>0$, employing position gradients and central yaw finite differences with fixed step target $\boldsymbol q_i$.

To keep Eqs.~\eqref{eq:collision_row}--\eqref{eq:link_row} linear in the QP variables, introduce epigraph scalars $z_{ij,a}\ge\pm(v_{i,a}-v_{j,a})$ for $a\in\{x,y\}$ and replace $L_{ij}$ with $z_{ij,x}+z_{ij,y}$, obtaining affine rows $\widetilde F_{ij}^{\rm c,l}\ge 0$.
Because all coefficients multiplying $z_{ij,a}$ are nonpositive, affine row feasibility implies satisfaction of Eqs.~\eqref{eq:collision_row}--\eqref{eq:link_row} in exact arithmetic.
Each UAV $i\in\mathcal S$ solves its local convex QP with control $\boldsymbol u_i=[v_{i,x},v_{i,y},\omega_i]^T$:
\begin{equation}
\begin{aligned}
\min_{\boldsymbol u_i,\boldsymbol z_i,\boldsymbol\delta_i}\quad
 &\|\boldsymbol u_i\|^2+k_\delta(\delta_{i,d}+\delta_{i,\psi})\\
\text{subject to}\quad
 &\widetilde F_{ij}^{\rm c}\ge0,\quad j<i,\\
 &\widetilde F_{ij}^{\rm l}\ge0,\quad j\in\mathcal R_i^{\rm loc},\\
 &z_{ij,a}\ge\pm(v_{i,a}-v_{j,a}),\quad a\in\{x,y\},\\
 &\boldsymbol a_{i,\lambda}^T\boldsymbol u_i+c_{i,\lambda}\ge-\delta_{i,\lambda},\quad \lambda\in\{d,\psi\},\ \boldsymbol\delta_i\ge0,\\
 &|v_{i,a}|\le U_i,\quad |\omega_i|\le\bar\omega_i.
\end{aligned}
\label{eq:local_qp}
\end{equation}

Each sUAV solves its local QP from exchanged state estimates, posterior covariances, and current-step preceding commands in topological order without centralized optimization before synchronized actuation over $\Delta t$, as summarized in Algorithm~\ref{alg:search-update}.

\begin{algorithm}[!t]
\caption{Uncertainty-aware search update}
\label{alg:search-update}
\begin{algorithmic}[1]
\REQUIRE Grid $\mathcal G$, references $\mathcal R_i^{\rm loc}$, time budget $T_{\max}$, sensing and control parameters.
\STATE Initialize step $k\leftarrow 0$, online task map, estimates $\hat{\boldsymbol p}_i$, and covariances $\boldsymbol P_i$.
\WHILE{$\mathcal U_k > 0$ \textbf{and} $k\Delta t < T_{\max}$}
\STATE Predict pose and covariance from previous commands via Eq.~\eqref{eq:ekf_prediction}.
\FOR{each sUAV in increasing index order}
\STATE Fuse causal ranges via Eqs.~\eqref{eq:sci_update}--\eqref{eq:sci_joseph}; publish posterior.
\ENDFOR
\STATE Compute and share uncertainty scales $\epsilon_i^k$ via Eq.~\eqref{eq:node_scale}.
\STATE Update online map via Eq.~\eqref{eq:shrunk_sector}; compute targets via Eqs.~\eqref{eq:spacing}--\eqref{eq:target_offset}.
\FOR{each sUAV $i$ in increasing index order}
\STATE Incorporate fixed same-step preceding commands $\boldsymbol v_j^k$ for $j<i$.
\STATE Construct signed bounds Eq.~\eqref{eq:snapshot_barriers} and solve local QP Eq.~\eqref{eq:local_qp}.
\STATE Publish current-step command $\boldsymbol u_i^k$.
\ENDFOR
\STATE Actuate the synchronized command batch simultaneously for $\Delta t$.
\STATE Update online unsearched count $\mathcal U_{k+1}$ via Eq.~\eqref{eq:unsearched_count}; set $k\leftarrow k+1$.
\ENDWHILE
\end{algorithmic}
\end{algorithm}

\Needspace{5\baselineskip}
\subsection{Conditional Continuous-Time Invariance Under Sampled Actuation}
\label{sec:preservation}
Because the distributed local QPs enforce inequalities Eqs.~\eqref{eq:collision_row}--\eqref{eq:link_row} only at discrete sampling instants $t_k$ using uncertain state estimates, we next establish conditions under which applying the resulting synchronized commands over $[t_k,t_k+\Delta t]$ maintains continuous-time safe-set invariance.
\begin{proposition}[Inter-sample safe-set invariance]
\label{prop:held-distance}
Suppose at sampling instant $t_k$, true inter-UAV separations and assigned reference distances satisfy Eq.~\eqref{eq:problem_distances}, instantaneous estimation errors satisfy $\|\boldsymbol r_{ij}(t_k)-\hat{\boldsymbol r}_{ij}^k\|\le \epsilon_{ij}^k$ for every constrained pair, and $s_{ij}\ge \epsilon_{ij}^k$ for all collision pairs.
If the final synchronized commands satisfy actuator limits, hard constraints Eqs.~\eqref{eq:collision_row}--\eqref{eq:link_row} evaluated with the same-step estimates, scales, and frozen preceding commands, and piecewise-constant actuation Eq.~\eqref{eq:dynamics} is applied over $[t_k,t_k+\Delta t]$ with $\alpha_{\rm c}\Delta t\le 1$ and $\alpha_{\rm l}\Delta t\le 1$, then $h_{ij}^{\rm c}(t)\ge 0$ and $h_{ij}^{\rm l}(t)\ge 0$ remain satisfied continuously for all $t\in[t_k,t_k+\Delta t]$.
\end{proposition}

\emph{Proof:}
Suppress pair subscripts and write $\boldsymbol w=\boldsymbol v_i-\boldsymbol v_j$, $L=\|\boldsymbol w\|_1$, and $B=U_i+U_j$.
By the triangle inequality and Cauchy--Schwarz with $\|\boldsymbol w\|_2\le L$ under $s\ge \epsilon$ for collision pairs, the snapshot bounds obey $\underline h^{\rm c}\le h^{\rm c}(t_k)$ and $\underline h^{\rm l}\le h^{\rm l}(t_k)$, while initial estimation error gives $|(\boldsymbol r(t_k)-\hat{\boldsymbol r})^T\boldsymbol w|\le \epsilon L$.
Substituting these inequalities into Eqs.~\eqref{eq:collision_row}--\eqref{eq:link_row} produces
\begin{equation}
\begin{aligned}
2\boldsymbol r(t_k)^T\boldsymbol w&\ge-\alpha_{\rm c}\underline h^{\rm c},\\
-2\boldsymbol r(t_k)^T\boldsymbol w&\ge-\alpha_{\rm l}\underline h^{\rm l}+\Delta t BL.
\end{aligned}
\label{eq:initial_derivative_bounds}
\end{equation}

For $\tau\in[0,\Delta t]$, piecewise-constant actuation Eq.~\eqref{eq:dynamics} implies exact displacement $\boldsymbol r(t_k+\tau)=\boldsymbol r(t_k)+\tau\boldsymbol w$.
Furthermore, componentwise limits $|w_a|\le U_i+U_j=B$ ensure $\|\boldsymbol w\|_2^2\le BL$.
Expanding the squared true distance over $\tau\in[0,\Delta t]$ gives for collision pairs:
\begin{align}
h^{\rm c}(t_k+\tau)
&=h^{\rm c}(t_k)+2\tau\boldsymbol r(t_k)^T\boldsymbol w+\tau^2\|\boldsymbol w\|_2^2\notag\\
&\ge h^{\rm c}(t_k)-\alpha_{\rm c}\tau\underline h^{\rm c}+\tau^2\|\boldsymbol w\|_2^2\notag\\
&\ge(1-\alpha_{\rm c}\tau)h^{\rm c}(t_k)\ge0,
\label{eq:collision_hold_proof}
\end{align}
since $\alpha_{\rm c}\tau\le\alpha_{\rm c}\Delta t\le 1$ and $h^{\rm c}(t_k)\ge 0$.
Similarly, for assigned reference pairs, the inward quadratic displacement enters with a minus sign:
\begin{align}
h^{\rm l}(t_k+\tau)
&=h^{\rm l}(t_k)-2\tau\boldsymbol r(t_k)^T\boldsymbol w-\tau^2\|\boldsymbol w\|_2^2\notag\\
&\ge h^{\rm l}(t_k)-\alpha_{\rm l}\tau\underline h^{\rm l}+\tau(\Delta t BL-\tau\|\boldsymbol w\|_2^2)\notag\\
&\ge h^{\rm l}(t_k)-\alpha_{\rm l}\tau\underline h^{\rm l}+\tau(\Delta t-\tau)BL\notag\\
&\ge(1-\alpha_{\rm l}\tau)h^{\rm l}(t_k)\ge0,
\label{eq:link_hold_proof}
\end{align}
because $\tau\le \Delta t$ ensures $\tau(\Delta t-\tau)BL\ge 0$ and $\alpha_{\rm l}\tau\le 1$.
Hence both distance constraints remain satisfied for all $t\in[t_k,t_k+\Delta t]$.\hfill$\square$

\begin{remark}[Conditional invariance, probabilistic scale, and containment metric]
\label{rem:conditional}
Proposition~\ref{prop:held-distance} is a conditional invariance statement rather than an unconditional safety claim: it certifies continuous-time separation and reference bounds over $[t_k,t_k+\Delta t]$ strictly on steps where instantaneous estimation errors obey $\|\boldsymbol r_{ij}(t_k)-\hat{\boldsymbol r}_{ij}^k\|\le\epsilon_{ij}^k$, each local QP Eq.~\eqref{eq:local_qp} is feasible, and commands are actuated synchronously.
Because all bounds are evaluated from step-start snapshots, discrete covariance jumps at measurement updates require no time differentiation, and each posterior reset re-initializes the check for step $k+1$ without relaxing hard constraints when actuator limits bind.
Because SCI processes stochastic ranges, $\epsilon_i^k$ is a probabilistic design scale rather than a deterministic envelope: under SCI consistency $\mathbb E[\boldsymbol e_i\boldsymbol e_i^T]\preceq\boldsymbol P_i^k$, Markov's inequality yields the distribution-free per-step exceedance bound $\mathbb P(\|\boldsymbol e_i\|>\epsilon_i^k)\le 1/9$, tightening to at most $e^{-9/2}\approx 1.11\%$ under a 2-D Gaussian distribution.
Accordingly, $C_{\rm env}$ in Sections~\ref{sec:results}--\ref{sec:gazebo} measures the per-sUAV sample rate satisfying $\|\boldsymbol p_i-\hat{\boldsymbol p}_i\|\le\epsilon_i^k$, implying Proposition~\ref{prop:held-distance}'s pairwise hypothesis via the triangle inequality when both endpoints are contained, alongside local QP feasibility.
\end{remark}

\begin{remark}[Idealized kinematics, communication, and delay-compensated CBFs]
\label{rem:assumptions}
Proposition~\ref{prop:held-distance} relies on the nominal assumptions of first-order kinematics Eq.~\eqref{eq:dynamics} and synchronized zero-delay inter-UAV exchange so that preceding same-step commands $\boldsymbol v_j^k$ arrive before actuation.
Standard sampled-data CBFs do not inherently compensate unbounded communication or input delays, and uncompensated sequential per-hop delays would accumulate along a multi-hop chain.
Formally certifying invariance under known delay bounds requires delay-compensated sampled-data or predictor-based CBF formulations~\cite{capelli_connectivity_delay_2021}.
To preserve the lightweight local QP structure, this paper retains the uncompensated nominal controller and directly evaluates its empirical degradation under multirotor flight dynamics, packet loss, and 20--200~ms uncompensated latency in Section~\ref{sec:gazebo}.
\end{remark}

\section{Numerical Evaluation}
\label{sec:results}
\subsection{Setup and Metrics}
\label{sec:setup}
The numerical configuration uses a $3000\times 3000$~m region covering $[-1500,1500]^2$~m, discretized into 90,000 grid cells at 10~m spacing, supported by coastal bUAVs at $(750,-1550)$, $(0,-1550)$, and $(-750,-1550)$~m, with the middle beacon shared between two squads.
In the 14-sUAV baseline, odd-indexed members of the first squad initialize at $(120+20k,-1490+40k)$~m for $k=0,\ldots,3$ and even-indexed members at $(160+20k,-1490+40k)$~m for $k=0,\ldots,2$, with the second squad mirrored across the $y$-axis.
Filters initialize from known deployment coordinates with $\boldsymbol P_i=\boldsymbol I$~m$^2$.
System, sensing, ranging, and control parameters follow Table~\ref{tab:parameters} at sampling interval $\Delta t=0.1$~s, capped at a finite 700~s budget, and the local QP is solved with PIQP~\cite{schwan2023piqp}. The source code is publicly available\footnote{\url{https://github.com/xirhxq/cbf/tree/ugdms-cbf}}.
All internal and baseline comparisons use matched deployment and seeded ranging fields.

Following Algorithm~\ref{alg:search-update}, each run terminates strictly from the online task map once $\mathcal U_k=0$, or once the first sUAV's local online map reaches 100\% under asynchronous messaging in Section~\ref{sec:gazebo}. Offline post-processing then verifies whether team true coverage also reaches 100\% at termination.
Localization performance is evaluated across all executed sUAV control instants via position root-mean-square error (RMSE), peak error-to-scale ratio $\|\boldsymbol p_i-\hat{\boldsymbol p}_i\|/\epsilon_{i,\rm ctrl}$, and empirical error containment $C_{\rm env}$, defined as the percentage of executed per-sUAV control-step samples satisfying $\|\boldsymbol p_i-\hat{\boldsymbol p}_i\|\le \epsilon_{i,\rm ctrl}$, where $\epsilon_{i,\rm ctrl}$ equals $\epsilon_i^k$ under dynamic scaling or the prescribed constant margin under fixed-scale ablations.
Constraint compliance checks solver feasibility alongside continuous-time separation and reference-distance bounds.

\begin{table}[!t]
\caption{Simulation and Control Parameters}
\label{tab:parameters}
\centering
\begin{tabular}{lc}
\hline
Parameter & Value\\\hline
Main-case sUAVs / bUAVs & 14 / 3\\
Initial yaw & $90^\circ$\\
Sensing radius $r_{\rm det}$ / field of view $\theta_{\rm hfov}$ & 400~m / $120^\circ$\\
Ranging parameters $d_{\rm r}$, $\sigma_0$ & 850~m, 0.5~m\\
Operational ranging cutoff $d_{\max}$ & 1200~m\\
Control step period $\Delta t$ & 0.1~s\\
CBF gains $\alpha_{\rm c}$, $\alpha_{\rm l}$ & 0.1~s$^{-1}$\\
Componentwise speed limit $U_i$ & 25~m/s\\
Yaw-rate limit $\bar\omega_i$ & 0.35~rad/s\\
Minimum separation $d_{\rm safe}$ / reference limit $d_{\rm loc}$ & 10~m / 850~m\\
Task parameters $k_p$, $k_\psi$, $k_\delta$ & 5, 1, 10\\
Position / yaw task CBF gains $\alpha_d$, $\alpha_\psi$ & 1 / 1~s$^{-1}$\\
Initial covariance $\boldsymbol P_i^0$ / process noise $\sigma_{\rm w}$ & $\boldsymbol I$~m$^2$ / 1~m/s\\
Nominal flight altitude & 200~m\\
Uncertainty scale $\epsilon_i$ & $3\sqrt{\operatorname{tr}\boldsymbol P_i}$\\
\hline
\end{tabular}
\end{table}

\begin{figure*}[!t]
\centering
\subfloat[0.0~s, 4.60\%]{\includegraphics[width=0.162\textwidth]{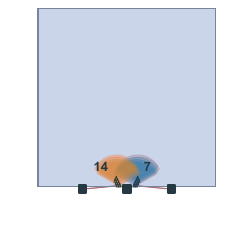}}\hfill
\subfloat[50.0~s, 29.70\%]{\includegraphics[width=0.162\textwidth]{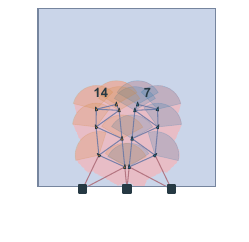}}\hfill
\subfloat[100.0~s, 59.33\%]{\includegraphics[width=0.162\textwidth]{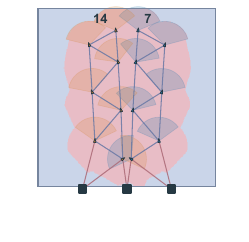}}\hfill
\subfloat[150.0~s, 86.29\%]{\includegraphics[width=0.162\textwidth]{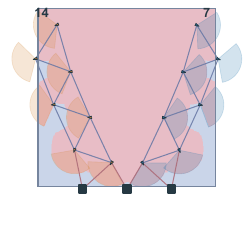}}\hfill
\subfloat[200.0~s, 99.62\%]{\includegraphics[width=0.162\textwidth]{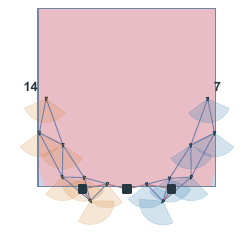}}\hfill
\subfloat[221.1~s, 100\%]{\includegraphics[width=0.162\textwidth]{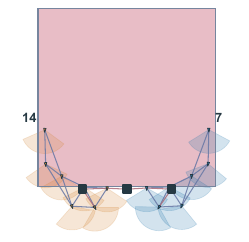}}
\caption{Main-case search progression for 14~sUAVs at key timestamps, showing accumulated true coverage in pink, unsearched area in light blue, forward sensing cones, coastal beacons as squares, and assigned reference chains.}
\label{fig:main-search}
\end{figure*}

\subsection{Main Case and Scalability}
\label{sec:main-results}
\label{sec:scenarios}
The 14-sUAV main case evaluates whether the hierarchical ladder and SCI-scaled constraints coordinate rapid offshore search without violating assigned reference-distance bounds or inter-UAV separation.
Full online and true coverage are achieved in 221.1~s, shown in Figs.~\ref{fig:main-search}--\ref{fig:main-diagnostics}, with the two squads searching symmetrically outward from the coastline before covering inland margins together.
Despite multi-hop ranging propagation into offshore waters, distance SCI suppresses shared-ancestor cross-correlation growth and maintains localization accuracy at 0.45~m RMSE, achieving $C_{\rm env}=100\%$ containment with a 0.80 peak error-to-scale ratio.
True inter-UAV distances stay strictly inside the safe sets throughout every control interval, peaking at 848.13~m along assigned reference edges below $d_{\rm loc}$ and maintaining at least 40.00~m inter-UAV separation above $d_{\rm safe}$ with 100\% local QP feasibility throughout.
As shown in Fig.~\ref{fig:main-diagnostics}(d), inter-UAV separation stays above 40.00~m not because the collision barrier is inactive, but because Eq.~\eqref{eq:collision_row} brakes approaching pairs smoothly at long range under $\alpha_{\rm c}=0.1\text{ s}^{-1}$: for instance, sUAV~8 relative to sUAV~1 slides along the uncertainty-inflated CBF closing-speed limit for 183 consecutive active steps from 0.0 to 18.2~s before diverging well ahead of $d_{\rm safe}=10$~m.

\begin{figure}[!t]
\centering
\subfloat[]{\includegraphics[width=\columnwidth]{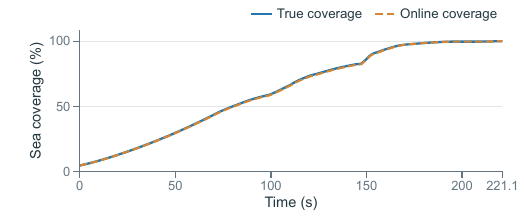}}\\[2pt]
\subfloat[]{\includegraphics[width=\columnwidth]{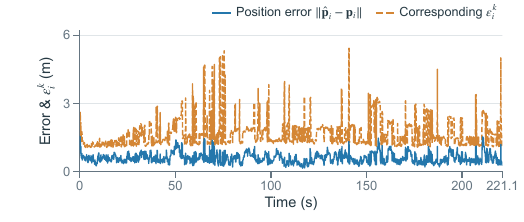}}\\[2pt]
\subfloat[]{\includegraphics[width=\columnwidth]{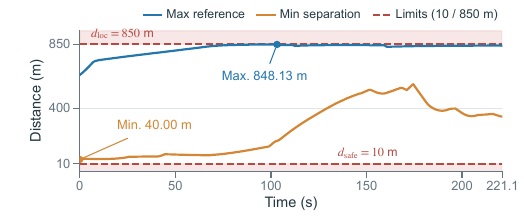}}\\[2pt]
\subfloat[]{\includegraphics[width=\columnwidth]{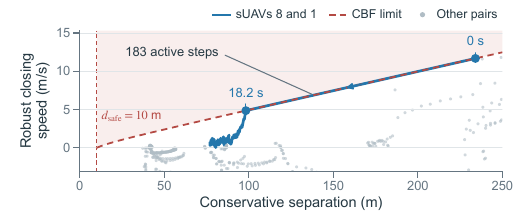}}
\caption{Main-case diagnostics across the 221.1~s mission: (a) online and true area coverage; (b) largest individual estimation error and corresponding uncertainty scale $\epsilon_i$; (c) extreme pairwise reference distances and inter-UAV separations relative to safety limits; (d) collision-CBF robust closing speed versus conservative separation over 0--40~s.}
\label{fig:main-diagnostics}
\end{figure}

To examine geometric and team-size scalability, Table~\ref{tab:scenarios} reports four additional convex domain shapes alongside a 10-sUAV two-squad team and a 21-sUAV three-squad team with four coastal beacons, while Fig.~\ref{fig:n10-stall} visualizes the non-square domains and contrasting swarm sizes.
All five 14-sUAV domains, spanning squares from 2000~m to 3000~m alongside triangle and trapezoid shapes, finish in 107.9--221.1~s with 100\% coverage, zero error-scale exceedances, and continuous constraint satisfaction.
Comparing 10 and 21~sUAVs clarifies the distinction between optimization feasibility and domain reachability: both swarms remain 100\% QP-feasible throughout their horizons, and adding a third seven-member squad completes the 3000~m square faster at 219.7~s with 30.11~m minimum separation.
By contrast, two five-member chains lack sufficient hop span to reach the upper boundary under $d_{\rm loc}=850$~m, arresting offshore progression near mid-domain at 57.52\% true coverage over the 700~s horizon.
Consequently, offshore reachability under finite ranging depends on matching per-squad chain length and parallel squad count to the depth and width of the search region: when a smaller swarm saturates below full coverage due to physical chain-length limits rather than optimization failure, enlarging the team restores complete coverage.

\begin{table}[!t]
\caption{Region and Swarm-Size Scalability}
\label{tab:scenarios}
\centering
\setlength{\tabcolsep}{4pt}
\begin{tabular}{@{}lcccc@{}}
\hline
Configuration, $N$ & \begin{tabular}[t]{@{}c@{}}End\\(s)\end{tabular} & \begin{tabular}[t]{@{}c@{}}Coverage\\(\%)\end{tabular} & \begin{tabular}[t]{@{}c@{}}Max reference\\(m)\end{tabular} & \begin{tabular}[t]{@{}c@{}}Min separation\\(m)\end{tabular}\\\hline
Main square, 14 & 221.1 & 100.00 & 848.13 & 40.00\\
2000~m square, 14 & 107.9 & 100.00 & 819.40 & 40.00\\
2500~m square, 14 & 175.9 & 100.00 & 844.71 & 40.00\\
Triangle, 14 & 189.1 & 100.00 & 848.10 & 40.00\\
Trapezoid, 14 & 204.1 & 100.00 & 848.13 & 40.00\\
Two squads, 10 & 700.0 & 57.52 & 846.72 & 40.00\\
Three squads, 21 & 219.7 & 100.00 & 847.76 & 30.11\\\hline
\end{tabular}
\end{table}

\begin{figure}[!t]
\centering
\subfloat[Triangle]{\includegraphics[width=0.242\columnwidth]{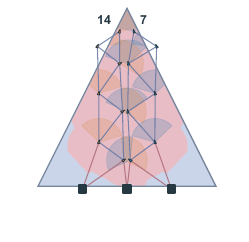}}\hfill
\subfloat[Trapezoid]{\includegraphics[width=0.242\columnwidth]{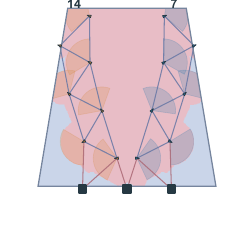}}\hfill
\subfloat[$N=10$]{\includegraphics[width=0.242\columnwidth]{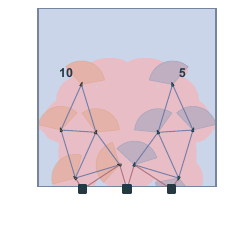}}\hfill
\subfloat[$N=21$]{\includegraphics[width=0.242\columnwidth]{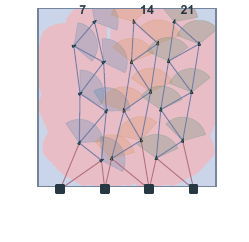}}
\caption{Representative spatial search snapshots across contrasting domain geometries and swarm sizes: (a) 14-sUAV triangle search progress; (b) 14-sUAV trapezoid search progress; (c) 10-sUAV two-squad reach saturation at 57.52\% coverage within $d_{\rm loc}=850$~m; (d) 21-sUAV three-squad search progress on the 3000~m square.}
\label{fig:n10-stall}
\end{figure}

\subsection{Comparisons}
\label{sec:comparisons}
We evaluate both internal controller components and four external target/safety baseline adaptations on the 14-sUAV baseline setup under identical initial positions, sensing, SCI estimation except the $c_r=3$ extended Kalman filter (EKF) variant, input limits, sampling period, and ranging fields.
Internal ablations in Table~\ref{tab:ablation} isolate adaptive covariance margins and correlation-aware estimation under both nominal ranging and degraded over-water ranging at $\sigma_0=1$~m and availability scale $d_{\rm r}=600$~m.
Fixed control scales expose the calibration dilemma of static margins: under nominal ranging, a small 1~m margin shortens completion to 219.4~s only by lowering $C_{\rm env}$ to 96.84\% with a 2.98 peak error-to-scale ratio, whereas a large 9~m margin contracts reference edges down to 840.31~m and slows completion to 223.0~s.
While fixed 3~m and 9~m bounds cover the nominal trajectory with $C_{\rm env}=100\%$, static margins cannot recalibrate when channel quality worsens: under degraded ranging, localization errors climb to 3.60~m and drop the uncalibrated fixed 3~m margin to $C_{\rm env}=99.79\%$ with a 1.20 peak ratio, whereas dynamic SCI scaling $\epsilon_i^k$ widens up to 18.01~m to maintain $C_{\rm env}=100\%$ with a 0.69 peak ratio in 222.6~s.

Compared with empirical EKF variance inflation using $c_r=3$, SCI prevents double-counting of shared-ancestor errors by assigning $\omega=0$ to strongly correlated neighbor ranges.
Fusing all correlated ranges without correlation rejections trims nominal EKF RMSE by 0.04~m from 0.45 to 0.41~m, yet underestimates covariance confidence: while the conservative outer circle $\epsilon_i^k$ still masks this under-dispersion in Table~\ref{tab:ablation}, EKF raises the peak error-to-scale ratio above SCI nominally at 0.83 versus 0.80 and more sharply under degradation at 0.89 versus 0.69, while exceeding its tighter $3\sigma$ ellipse major radius $3\sqrt{\lambda_{\max}(\boldsymbol P_i^k)}$ 18 times nominally and 57 times under degradation versus zero for SCI.
Furthermore, omitting $\Delta t B_{ij}L_{ij}$ from Eq.~\eqref{eq:link_row} yields virtually identical trajectories at 221.1~s, 0.45~m RMSE, and 848.21~m maximum reference, yet satisfies Proposition~\ref{prop:held-distance}'s full sufficient condition in only 29.0\% of steps versus 100\% when retained. Keeping $\Delta t B_{ij}L_{ij}$ thus preserves conditional continuous-time invariance at zero completion-time penalty.

\begin{table}[!t]
\caption{Internal Control-Scale and Estimator Ablations}
\label{tab:ablation}
\centering
\setlength{\tabcolsep}{3.5pt}
\begin{tabular}{@{}lccccc@{}}
\hline
Configuration & \begin{tabular}[t]{@{}c@{}}End\\(s)\end{tabular} & \begin{tabular}[t]{@{}c@{}}RMSE\\(m)\end{tabular} & \begin{tabular}[t]{@{}c@{}}$C_{\rm env}$\\(\%)\end{tabular} & \begin{tabular}[t]{@{}c@{}}Peak\\ratio\end{tabular} & \begin{tabular}[t]{@{}c@{}}Max reference\\(m)\end{tabular}\\\hline
SCI, covariance scale & 221.1 & 0.45 & 100.00 & 0.80 & 848.13\\
SCI, fixed 1~m scale & 219.4 & 0.47 & 96.84 & 2.98 & 848.48\\
SCI, fixed 3~m scale & 220.0 & 0.47 & 100.00 & 0.84 & 846.41\\
SCI, fixed 9~m scale & 223.0 & 0.45 & 100.00 & 0.23 & 840.31\\
EKF with $c_r=3$ & 219.6 & 0.41 & 100.00 & 0.83 & 848.27\\\hline
\multicolumn{6}{@{}l@{}}{\emph{Degraded over-water ranging, $\sigma_0=1$~m, $d_{\rm r}=600$~m}}\\
SCI, covariance scale & 222.6 & 0.70 & 100.00 & 0.69 & 847.42\\
SCI, fixed 3~m scale & 220.2 & 0.70 & 99.79 & 1.20 & 846.49\\
EKF with $c_r=3$ & 220.0 & 0.68 & 100.00 & 0.89 & 847.76\\\hline
\end{tabular}
\end{table}

\begin{table*}[!t]
\caption{Five-Method Comparison with Common $0$--$220.2$~s Window RMSE}
\label{tab:adapted-comparison}
\centering
\begin{tabular}{lcccccc}
\hline
Method & \begin{tabular}[t]{@{}c@{}}End\\(s)\end{tabular} & \begin{tabular}[t]{@{}c@{}}True coverage\\(\%)\end{tabular} & \begin{tabular}[t]{@{}c@{}}RMSE\\(m)\end{tabular} & \begin{tabular}[t]{@{}c@{}}Max reference\\(m)\end{tabular} & \begin{tabular}[t]{@{}c@{}}Min separation\\(m)\end{tabular} & Outcome\\\hline
Proposed sequential local QPs & 221.1 & 100.00 & 0.45 & 848.13 & 40.00 & Complete\\
Sampled robust joint CBF & 220.2 & 100.00 & 0.48 & 844.63 & 40.00 & Complete\\
Localizability targets + joint CBF & 264.6 & 100.00 & 0.48 & 843.41 & 40.00 & Complete\\
Artificial-reference tube MPC & 274.7 & 100.00 & 0.52 & 837.67 & 37.41 & Complete\\
Distance-formation targets + joint CBF & 700.0 & 55.88 & 0.41 & 844.94 & 35.90 & Budget reached\\\hline
\end{tabular}
\end{table*}

We next compare four external controller and target adaptations documented with their modified interfaces:
\begin{itemize}
\item \emph{Sampled robust joint CBF}~\cite{oruganti_sampled_2024} retains our hierarchical targets while replacing Eq.~\eqref{eq:local_qp} by a centralized joint second-order cone program imposing the Lipschitz-tightened condition
\begin{equation}
\sigma\,2\hat{\boldsymbol r}^{T}\boldsymbol w+\alpha h
-2\rho\|\boldsymbol w\|_2-2\alpha M\rho\ge0,
\label{eq:oruganti_adaptation}
\end{equation}
with $\rho=\epsilon+\Delta t\sqrt2 B$, $M=s+\rho$, and $\sigma=\pm 1$ for collision and reference barriers.
\item \emph{Localizability targets + joint CBF}~\cite{cano_localizability_2023} keeps the leader rule and Eq.~\eqref{eq:oruganti_adaptation}, updating follower targets every 0.1~s using range-information, connectivity, and collision potentials under a 0.5~m nominal range standard deviation projected inside $\Omega$.
\item \emph{Distance-formation targets + joint CBF}~\cite{goudar_onboard_2026} pairs Eq.~\eqref{eq:oruganti_adaptation} with a 10-step, 1~s Gaussian-process distance-factor prediction targeting 750~m inter-UAV spacing.
\item \emph{Artificial-reference tube MPC}~\cite{nikou2019TubeDMPC,limon2010RobustTracking}, a robust distributed MPC formulation, tracks our search targets via sequential 20-step, 1~s predictions optimizing an artificial reference plus 0.1~s ancillary feedback using an 8.49~m tube, 20.76~m/s nominal axis limit, 5~m terminal ball, and shifted backup plan.
\end{itemize}

\begin{figure}[!t]
\centering
\includegraphics[width=\columnwidth]{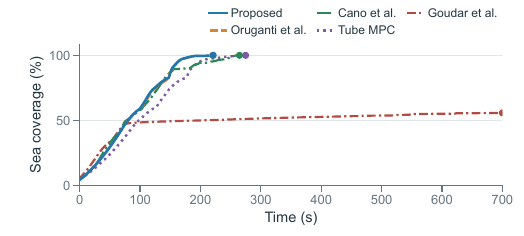}
\caption{True area coverage for the five compared 14-sUAV configurations.}
\label{fig:comparison}
\end{figure}

Table~\ref{tab:adapted-comparison} and Fig.~\ref{fig:comparison} summarize the five compared trajectories and common-window localization accuracy.
All five trajectories preserve pairwise distance bounds, separating cleanly across completion speed, safety margin, and estimation accuracy.
Without centralized full-swarm optimization, our decentralized sequential local QPs match the centralized sampled robust joint CBF in completion speed within 0.4\%, finishing in 221.1~s versus 220.2~s with lower common-window RMSE at 0.45~m versus 0.48~m, because both pair frontier leader search with coast-anchored follower geometry.
Compared with the likewise decentralized artificial-reference tube MPC, our single-step closed-form QP finishes 53.6~s earlier at 221.1~s versus 274.7~s and achieves tighter localization at 0.45~m versus 0.52~m RMSE, whereas tube MPC trades speed for a wider reference margin down to 837.67~m.
Pure localizability targets finish later at 264.6~s, while static distance-formation targets achieve 0.41~m RMSE only because followers cling within 0.6~m of near-shore formation points, stranding leaders 400.2~m behind search goals at 55.88\% coverage after 700~s.
Maritime swarm search therefore demands coordinating relay geometry directly with offshore leader progress rather than optimizing localization or static formation criteria in isolation, and our sequential local QPs deliver nearly centralized search efficiency through decentralized per-sUAV solves.

\subsection{Runtime Evaluation}
\label{sec:timing}
\begin{table}[!t]
\caption{Computation Timing Results}
\label{tab:timing}
\centering
\setlength{\tabcolsep}{4pt}
\begin{tabular}{lcc}
\hline
Platform, Role, $N$ & \begin{tabular}[t]{@{}c@{}}Control cycle (ms)\\(mean / p95)\end{tabular} & \begin{tabular}[t]{@{}c@{}}Full cycle (ms)\\(mean / p95)\end{tabular}\\\hline
M1 Pro, numerical, 14 & 0.84 / 1.26 & 0.93 / 1.38\\
Orin NX, onboard, 10 & 1.87 / 2.50 & 2.55 / 3.61\\
Orin NX, onboard, 14 & 2.29 / 3.29 & 3.11 / 4.74\\
Orin NX, onboard, 21 & 3.32 / 5.33 & 4.17 / 6.38\\\hline
\end{tabular}
\end{table}

Table~\ref{tab:timing} reports per-sUAV control-cycle and full-cycle execution times measured on an Apple M1 Pro numerical host at $N=14$ and an NVIDIA Jetson Orin NX, a typical UAV onboard computer, across $N\in\{10,14,21\}$, where p95 denotes the 95th percentile.
Each per-sUAV control cycle covers online target selection, CBF constraint assembly, and CPU-only PIQP optimization, while the full cycle adds the per-sUAV SCI covariance update. Within each control cycle, PIQP solves average 0.151~ms on the M1 Pro and 0.316--0.723~ms on the Orin NX.
On the Orin NX, full per-sUAV onboard cycles average 2.55--4.17~ms with worst-case maxima of 5.08--9.04~ms and zero steps exceeding the 100~ms control period across all evaluated swarm sizes.

In Eq.~\eqref{eq:local_qp}, each sUAV enforces strictly two fixed reference rows regardless of team size, so growth in QP variables and inequality constraints arises solely from the $i-1$ pairwise collision rows $j<i$.
At $N=10, 14,$ and $21$, the highest-indexed local QP involves at most 27, 35, and 49 scalar variables alongside 11, 15, and 22 hard CBF rows, and adding lateral seven-sUAV squads along a coast keeps each local QP dimension bounded independently of the added squads.
For still larger fleets, screening collision constraints to a finite physical proximity neighborhood caps per-sUAV QP complexity independently of $N$, leaving sequential intra-squad wireless communication as the primary scaling factor. This communication overhead can be mitigated by cross-squad parallelism and slightly staggered 10~Hz local clocks consuming timestamped preceding commands alongside the 20--50~ms delays evaluated in Section~\ref{sec:gazebo}.

\section{Gazebo Flight Simulation}
\label{sec:gazebo}
\begin{table*}[!t]
\caption{Maritime Physical, Communication, and Flight-Disturbance Models}
\label{tab:disturbance-models}
\centering
\begin{tabular}{@{}l@{\hspace{8pt}}p{.77\textwidth}@{}}
\hline
Factor & Applied model and parameters\\\hline
Over-water ranging degradation & $\sigma_0=1$~m, twice nominal; scale $d_{\rm r}$ reduced from 850 to 600~m; 1200~m cutoff~\cite{bellusci_uwb_2009}\\
Sea-surface multipath bias & Eq.~\eqref{eq:two-ray} at $H=50$ or $200$~m with 5\% per-leg mislock probability~\cite{dardari2009ranging,matolak2017}\\
bUAV drift & Reported bUAV positions with 1~m RMS, 30~s OU drift~\cite{uhlenbeck1930brownian}; $\boldsymbol P_b=0.5\boldsymbol I$~m$^2$, $\epsilon_b=3$~m\\
Ranging outage & All incoming ranges to sUAV~3 suppressed for 10~s\\
Communication loss & 10\% independent or Gilbert--Elliott bursts, $p_{G\to B}=1/45$, $p_{B\to G}=0.2$ per 0.1~s, 0.5~s mean; 5~s map sync~\cite{gilbert1960,elliott1963}\\
Communication delay & Uncompensated 20, 50, or 200~ms inter-UAV communication latency\\
Wind gusts and turbulence & 5~m/s background wind with 3~m/s, 6~s orthogonal gust; light Dryden turbulence at 0.88/0.77~m/s horizontal/vertical RMS or moderate at 1.76/1.54~m/s~\cite{Madden_2018}\\
Yaw observation error & Zero-mean Gaussian yaw error with $\sigma=2^\circ$ or $5^\circ$~\cite{Guo_2025}\\\hline
\end{tabular}
\end{table*}
The analytical design and numerical evaluations in Sections~\ref{sec:model}--\ref{sec:results} rely on four idealized assumptions: first-order planar kinematics Eq.~\eqref{eq:dynamics}, zero-mean Gaussian ranging Eq.~\eqref{eq:ranging_model} without multipath bias or outages, stationary error-free coastal bUAVs with $\boldsymbol P_b=\boldsymbol 0$ and $\epsilon_b=0$, and synchronized zero-delay, loss-free inter-UAV communication.
To stress-test the proposed framework when every one of these assumptions is relaxed alongside multirotor flight physics, we integrate the rebuilt SCI estimator and sequential local QP controller into the MBZIRC ROS~2/Ignition Gazebo maritime simulator~\cite{mbzirc}, where 14~quadrotor sUAVs execute full second-order flight dynamics and asynchronous messaging under the maritime observation, communication, and flight disturbances summarized in Table~\ref{tab:disturbance-models}.
The source code is publicly available\footnote{\url{https://github.com/xirhxq/cbf-ros2/tree/ugdms-cbf}}.

\subsection{Setup and Disturbance Models}
\label{sec:gazebo-setup}
Each sUAV executes an outer velocity and yaw cascade controller atop the onboard inner flight loop and runs SCI motion prediction from onboard horizontal velocity feedback.
Because Eqs.~\eqref{eq:dynamics}--\eqref{eq:local_qp} synthesize piecewise-constant velocity commands under a first-order kinematic model, retaining $\alpha_d=1\text{ s}^{-1}$ demands abrupt acceleration, deceleration, and heading-reversal transients under which second-order quadrotor flight dynamics exhibit inner-loop tracking lag and overshoot.
As the sole algorithm adaptation for physical flight, reducing the soft position-task CBF gain $\alpha_d$ from $1$ to $0.1\text{ s}^{-1}$ shapes translational approach and braking slopes within multirotor acceleration limits while keeping all other estimator, sensing, speed, and safety parameters identical to Table~\ref{tab:parameters}.

Three classes of maritime disturbances listed in Table~\ref{tab:disturbance-models} are evaluated individually and in combination.
First, over-water observation degradation pairs increased ranging noise and reduced availability with sea-surface two-path reflection~\cite{matolak2017}.
For equal endpoint height $H$, reflected-path excess $\Delta(d,H)=\sqrt{d^2+4H^2}-d$ and per-leg UWB mislock indicators $I_1,I_2\sim\operatorname{Bernoulli}(0.05)$~\cite{dardari2009ranging} define added bias
\begin{equation}
 b_{\rm ray}=\tfrac12(I_1+I_2)\Delta(d,H).
\label{eq:two-ray}
\end{equation}

Evaluating both nominal $H=200$~m and low-altitude $H=50$~m flights probes both large-bias and small-bias reflection regimes.
With stationary physical beacons, reported coastal bUAV positions undergo Ornstein--Uhlenbeck (OU) drift~\cite{uhlenbeck1930brownian}, injecting nonzero $\boldsymbol P_b=0.5\boldsymbol I$~m$^2$ and $\epsilon_b=3$~m into SCI and CBF rows alongside a 10~s local outage on sUAV~3.
Second, communication faults impose independent or Gilbert--Elliott burst packet losses~\cite{gilbert1960,elliott1963} with periodic map synchronization, together with 20--200~ms uncompensated latency.
Third, aerodynamic and attitude disturbances apply steady wind, orthogonal gusts, continuous Dryden turbulence~\cite{Madden_2018}, and Gaussian yaw observation errors~\cite{Guo_2025} with parameters in Table~\ref{tab:disturbance-models}.

\subsection{Flight Validation and Disturbance Results}
\label{sec:gazebo-results}

\begin{figure*}[!t]
\centering
\subfloat[3D, 0.0~s]{\includegraphics[width=0.242\textwidth]{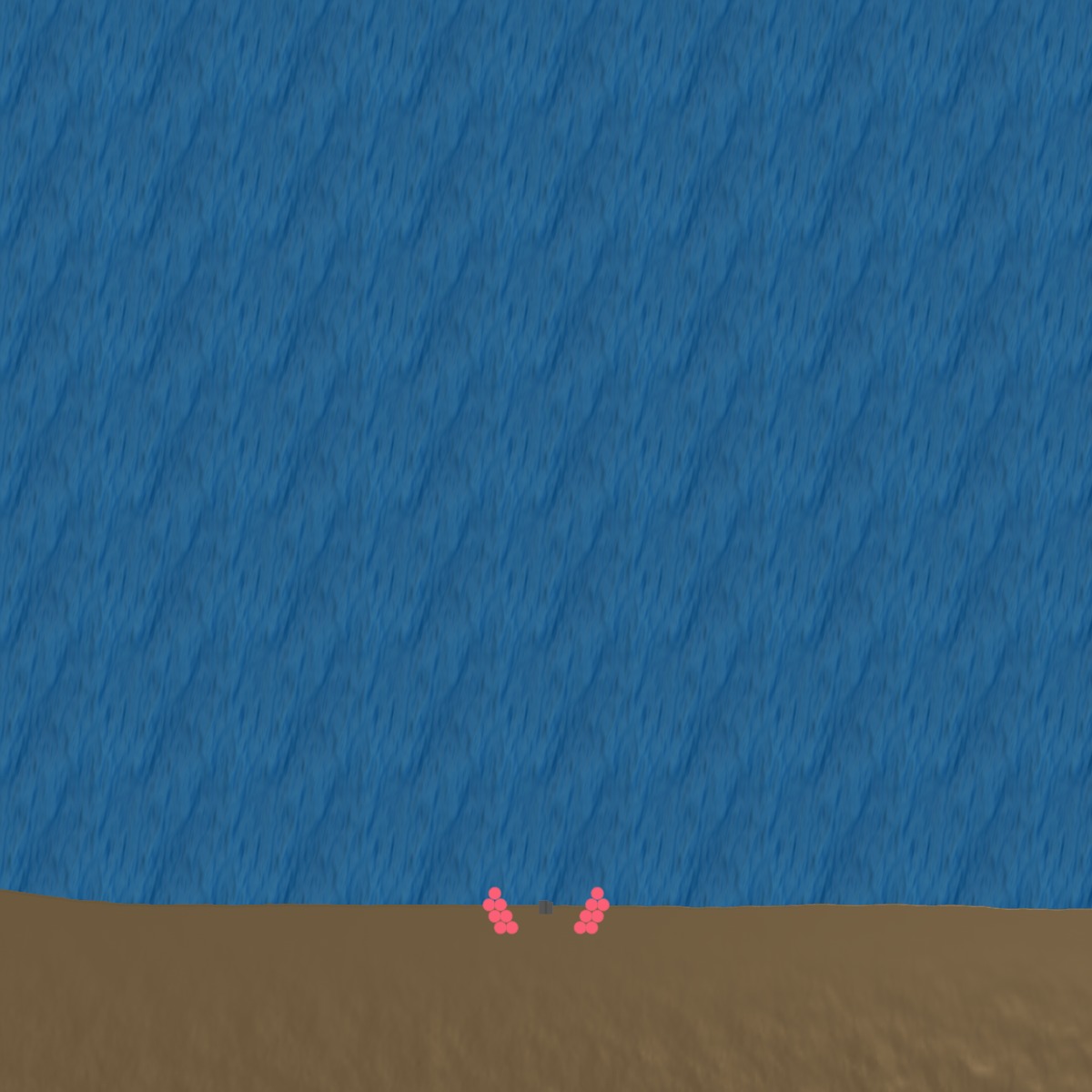}}\hfill
\subfloat[3D, 100.0~s]{\includegraphics[width=0.242\textwidth]{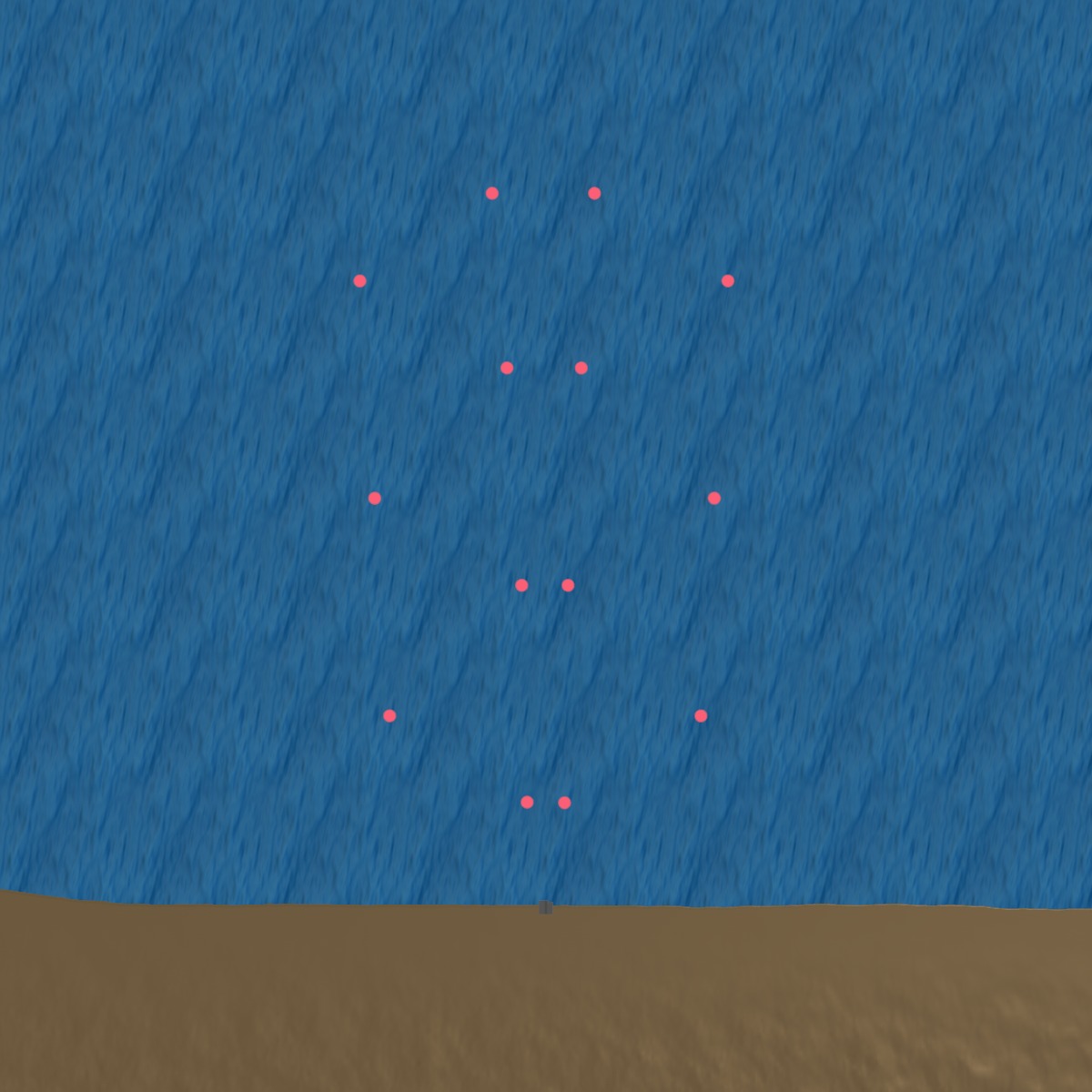}}\hfill
\subfloat[3D, 200.0~s]{\includegraphics[width=0.242\textwidth]{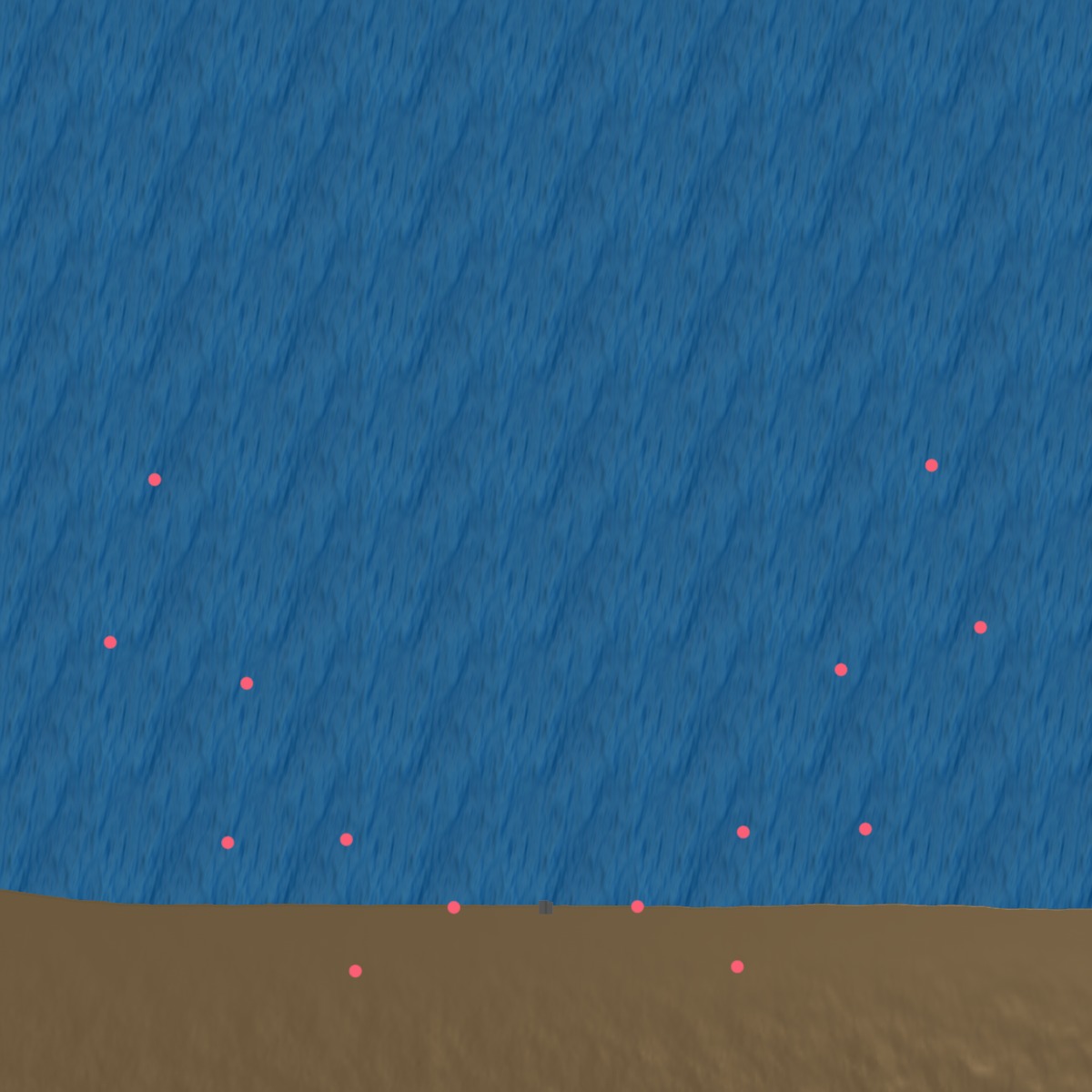}}\hfill
\subfloat[3D, 227.5~s]{\includegraphics[width=0.242\textwidth]{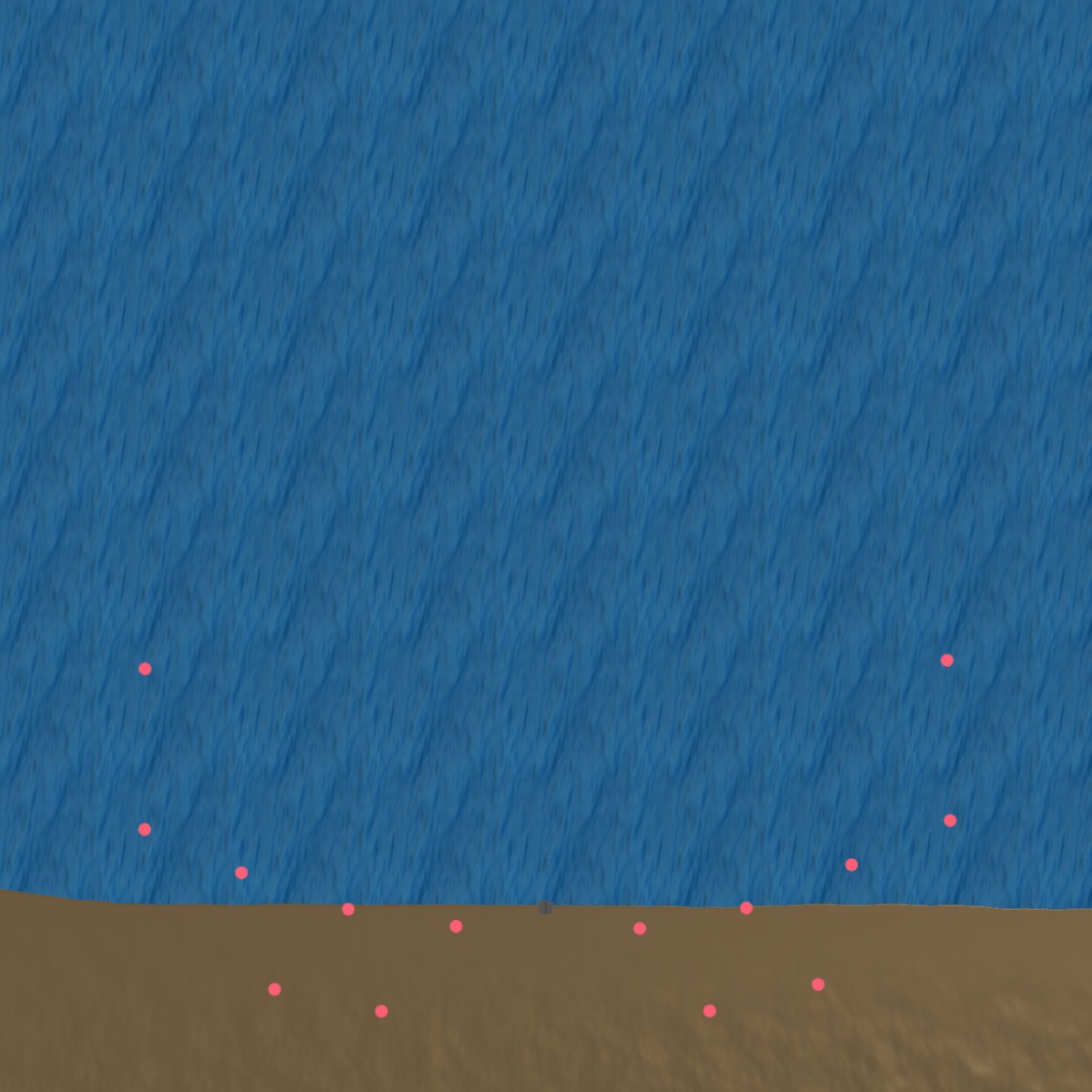}}\\[3pt]
\subfloat[Map, 0.0~s, 4.61\%]{\includegraphics[width=0.242\textwidth]{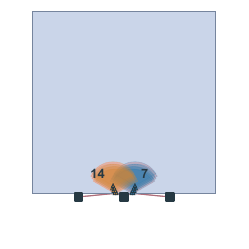}}\hfill
\subfloat[Map, 100.0~s, 57.30\%]{\includegraphics[width=0.242\textwidth]{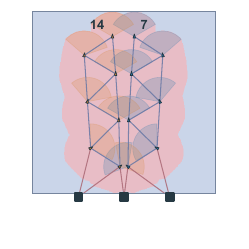}}\hfill
\subfloat[Map, 200.0~s, 99.43\%]{\includegraphics[width=0.242\textwidth]{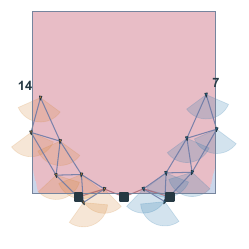}}\hfill
\subfloat[Map, 227.5~s, 100\%]{\includegraphics[width=0.242\textwidth]{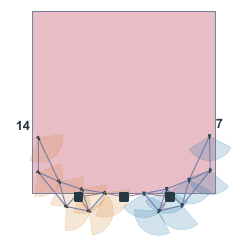}}
\caption{Baseline 14-sUAV Gazebo simulation: (a)--(d) top-down 3D scene snapshots showing coast, open sea, and red UAV positions; (e)--(h) synchronized true sensing map and assigned reference chains.}
\label{fig:gazebo-search}
\end{figure*}

\begin{table}[!t]
\caption{Gazebo Flight Simulation Results with Representative Disturbances}
\label{tab:gazebo}
\centering
\setlength{\tabcolsep}{4pt}
\begin{tabular}{@{}lccc@{}}
\hline
Configuration & \begin{tabular}[t]{@{}c@{}}End\\(s)\end{tabular} & \begin{tabular}[t]{@{}c@{}}RMSE\\(m)\end{tabular} & \begin{tabular}[t]{@{}c@{}}Max reference\\(m)\end{tabular}\\\hline
Baseline, 200~m & 227.5 & 0.61 & 847.90\\
Ranging noise doubled & 228.5 & 0.75 & 848.31\\
Availability scale 600~m & 228.2 & 0.66 & 847.48\\
Multipath bias, 50~m & 226.1 & 0.81 & 847.88\\
bUAV drift & 299.3 & 1.27 & 841.06\\
10~s ranging outage & 226.3 & 0.62 & 848.72\\
Packet loss, burst 10\% & 228.1 & 0.77 & 848.26\\
Communication delay, 50~ms & 227.0 & 1.62 & 849.49\\
Communication delay, 200~ms & 228.7 & 3.35 & 847.91\\
Dryden turbulence, moderate & 228.0 & 0.61 & 848.33\\
Wind + gust & 228.2 & 0.61 & 847.60\\
Yaw error, $5^\circ$ & 225.1 & 0.62 & 847.99\\
Combined disturbances, 20~ms & 302.0 & 1.68 & 840.95\\\hline
\end{tabular}
\end{table}

\begin{figure}[!t]
\centering
\includegraphics[width=\columnwidth]{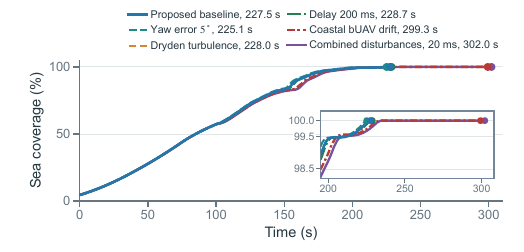}
\caption{True area coverage trajectories across representative Gazebo disturbance configurations.}
\label{fig:gazebo-disturbances}
\end{figure}

Fig.~\ref{fig:gazebo-search} illustrates the 200~m baseline flight, where reducing $\alpha_d$ to $0.1\text{ s}^{-1}$ for second-order dynamics shifts completion by 2.9\% relative to the kinematic simulation from 221.1 to 227.5~s at 0.61~m RMSE.
Table~\ref{tab:gazebo} and Fig.~\ref{fig:gazebo-disturbances} summarize 13 representative configurations from 21~total runs spanning all disturbance classes, with the remaining eight mild-intensity and single-factor variants exhibiting consistent trends.
Every run reaches full true coverage with minimum inter-UAV separation of at least 39.97~m and reference distances strictly under 850~m.

Under degraded over-water ranging, 10~s local outages, and sea-surface multipath, localization RMSE stays within 0.62--0.81~m and completion remains within 226.1--228.5~s, within roughly 1\% of the 227.5~s baseline.
The covariance-aware innovation gate rejects large geometric multipath biases at $H=200$~m, keeping RMSE at 0.64~m, whereas dynamic SCI uncertainty scaling widens reference margins to absorb smaller unrejected biases of 2.1--8.7~m at $H=50$~m, raising RMSE to 0.81~m compared with 0.61~m on the matched 50~m baseline, as well as transient prediction gaps.
When reported coastal beacon positions undergo OU drift, localization RMSE grows to 1.27~m and completion slows to 299.3~s. Injecting nonzero $\boldsymbol P_b=0.5\boldsymbol I$~m$^2$ and $\epsilon_b=3$~m causes SCI and reference CBF rows to pull outer limits inward to 841.06~m, preserving link integrity under drifting anchor reports.

Against asynchronous communication faults, both 10\% burst and independent packet losses add at most 0.8~s to mission duration.
Uncompensated delays of 20 and 50~ms both fall within a single control cycle and produce comparable single-step staleness near 1.62~m RMSE. The remaining 20~runs maintain $C_{\rm env}\ge 99.9\%$, whereas 200~ms crosses multiple cycles, raising RMSE to 3.35~m and lowering empirical containment to $C_{\rm env}=94.3\%$ without distance breaches.
Under steady wind, orthogonal gusts, light or moderate Dryden turbulence, and $2^\circ$ or $5^\circ$ yaw observation errors, completion remains within 225.1--228.2~s, again within roughly 1\% of baseline due to discrete frontier-cell switching, at 0.61--0.62~m RMSE because wind and yaw primarily perturb inner-loop flight tracking rather than range-based 2-D localization, while the smoothed position-task slope $\alpha_d=0.1\text{ s}^{-1}$ preserves thrust margin for wind rejection and distance--yaw task CBF decoupling prevents noisy yaw transients from distorting forward coverage trajectories.
Even when combining doubled noise, 600~m availability, 200~m multipath, bUAV drift, moderate Dryden turbulence, burst losses, and 20~ms delay simultaneously, the swarm finishes in 302.0~s at 1.68~m RMSE, and 250~Hz Gazebo records confirm zero separation or reference-distance violations across all flights.

Transitioning from Ignition Gazebo to physical maritime platforms still entails three practical engineering challenges.
First, real sea-surface ranging reliability faces dynamic wave scattering, airframe antenna occlusion, and intermittent non-line-of-sight biases beyond the two-ray model in Eq.~\eqref{eq:two-ray}, requiring adaptive outlier rejection and hardware-calibrated range-variance profiles.
Second, sequential intra-squad coordination over shared wireless channels must contend with medium-access contention, clock-synchronization drift, and multi-hop latency accumulation. Where end-to-end delays approach or exceed one control period, delay-compensated or predictor-based CBF updates~\cite{capelli_connectivity_delay_2021} are warranted alongside staggered 10~Hz schedules.
Third, estimator consistency on field hardware depends on bounding inertial-odometry velocity drift, sensor timestamp skew, and unmodeled wind or wave-induced coastal bUAV motion so that SCI covariances remain well calibrated when deploying the ROS~2 stack atop verified onboard inner-loop controllers.

\section{Conclusion}
\label{sec:conclusion}
This paper presented a CBF framework for GNSS-denied maritime swarm search under cascading position uncertainty and forward-looking sensing constraints, integrating coast-anchored hierarchical target geometry, SCI cooperative localization, and sampled-data robust CBF control.
By computing covariance-derived scales directly from distance SCI updates and accounting for inter-step relative displacement on assigned reference edges, distributed sequential local QPs maintain continuous-time separation and connectivity invariance across discrete control steps conditionally on step-wise error containment, local feasibility, and synchronization.
Numerical studies complete a 14-sUAV search in 221.1~s with 0.45~m RMSE, clarify reachability limits on a 10-sUAV team relative to 21-sUAV scaling, and establish quantitative trade-offs against sampled robust CBF, localizability, formation, and robust distributed tube-MPC baselines.
Tests on an NVIDIA Jetson Orin NX, a typical UAV onboard computer, verify 1.87--3.32~ms mean control cycles and 2.55--4.17~ms mean full per-sUAV onboard cycles across 10--21~sUAVs, while Ignition Gazebo quadrotor evaluations verify full mission completion across sea-surface multipath, bUAV drift, outages, wind gusts and turbulence, yaw errors, and communication faults with zero observed separation or reference-distance violations.

\bibliographystyle{IEEEtran}
\bibliography{references}
\label{manuscript-end}
\end{document}